\pdfoutput=1
\documentclass{arxivstyle}
\usepackage[utf8]{inputenc} 
\usepackage[T1]{fontenc}    
\usepackage{hyperref}       
\usepackage{url}            
\usepackage{booktabs}       
\usepackage{amsfonts}       
\usepackage{nicefrac}       
\usepackage{microtype}      

\usepackage{amsmath,amssymb,amsthm}
\usepackage{dsfont}
\usepackage{mathtools}

\usepackage{multirow}
\usepackage{nicematrix}
\usepackage{colortbl}

\usepackage{float}
\usepackage{graphicx}
\usepackage{subcaption}
\usepackage{caption}
\usepackage{wrapfig}
\usepackage{enumitem}

\usepackage{algorithm}

\usepackage{tikz}
\usepackage{pgfplots}
\pgfplotsset{compat=1.18}
\usepgfplotslibrary{groupplots}

\usepackage[most]{tcolorbox}

\usepackage{pifont}
\usepackage{inconsolata}
\usepackage{twemojis}

\usepackage{cleveref}

\definecolor{bg_c}{HTML}{F2F9FF}
\definecolor{bg_m}{HTML}{E6F4FF}
\definecolor{bg_h}{HTML}{CCE8FF}

\definecolor{BoxDarkBlue}{RGB}{38, 57, 77}
\definecolor{BoxLightBlue}{RGB}{232, 238, 245}
\definecolor{SoftGray}{RGB}{248, 249, 250}
\definecolor{CadetBlue}{RGB}{95, 158, 160}

\definecolor{DeepSlate}{RGB}{47, 79, 79}
\definecolor{GoldDetail}{RGB}{184, 134, 11}
\definecolor{MutedBlue}{RGB}{70, 130, 180}
\definecolor{DeepRed}{RGB}{165, 42, 42}

\definecolor{blue1}{HTML}{AEC7E8}
\definecolor{blue2}{HTML}{1F77B4}
\definecolor{orange1}{HTML}{FFBB78}
\definecolor{blue_c}{HTML}{2271B5}
\definecolor{blue_nc}{HTML}{9ABAD4}

\definecolor{ctx8k}{HTML}{1F77B4}
\definecolor{ctx16k}{HTML}{2CA02C}
\definecolor{ctx32k}{HTML}{D62728}
\definecolor{ctx64k}{HTML}{FF7F0E}

\definecolor{darkgreen}{HTML}{005e19}
\definecolor{darkblue}{HTML}{240394}
\definecolor{commentgray}{RGB}{80, 80, 80}

\newcommand{\cites}[1]{\citeauthor{#1}'s \citeyearpar{#1}}

\newcommand{\code}[1]{\texttt{#1}}

\newcommand{\examplegs}[1]{\textcolor{darkgreen}{\textbf{\scriptsize{\code{#1}}}}}
\newcommand{\exampleb}[1]{\textcolor{darkblue}{\textbf{\small{\code{#1}}}}}

\newcommand{\gold}{{\large\texttwemoji{1f642}}}
\newcommand{\lie}{{\large\texttwemoji{1f608}}}
\newcommand{\dstr}{{\large\texttwemoji{1f635-200d-1f4ab}}}

\DeclareRobustCommand{\cccbox}[1]{\raisebox{0.5\height}{\colorbox{#1}{\makebox(0.2ex,0.2ex){}}}}

\newcommand{\legmark}[1]{%
  \protect\tikz[baseline=-0.5ex]{%
    \protect\draw[#1, line width=1.2pt] (0,0)--(1.2em,0);%
    \protect\fill[#1] (0.6em,0) circle (1.8pt);%
  }%
}

\newtheorem{remark}{Remark}[section]

\newcommand{\papertitle}{Farther the Shift, Sparser the Representation: Analyzing OOD Mechanisms in LLMs}
\newcommand{\paperdate}{September 2026}
\newcommand{\codeurl}{https://github.com/MingyuJ666/sparsityLLM}

\begin{document}


\begin{herobox}
{\setlength{\parskip}{0pt}%

{\raggedright\titlefont\papertitle\par}

\vskip 0.2cm

{\raggedright\authorfont
Mingyu Jin\textsuperscript{1},
Yutong Yin\textsuperscript{2},
Jingcheng Niu\textsuperscript{3,$\dagger$},
Qingcheng Zeng\textsuperscript{2},
Wujiang Xu\textsuperscript{1},
Xinyuan Song\textsuperscript{7},
Mengnan Du\textsuperscript{4},
Wei Cheng\textsuperscript{5},
Zhaoran Wang\textsuperscript{2},
Tianlong Chen\textsuperscript{6},
Dimitris N Metaxas\textsuperscript{1}\par}

\vskip 0.2cm

{\raggedright\normalsize
\textsuperscript{1}Rutgers University,
\textsuperscript{2}Northwestern University,
\textsuperscript{3}University of Toronto,
\textsuperscript{4}New Jersey Institute of Technology,
\textsuperscript{5}NEC Laboratories America,
\textsuperscript{6}University of North Carolina at Chapel Hill,
\textsuperscript{7}Emory University\par
\vskip 0.1cm
{\small\textsuperscript{$\dagger$}Corresponding author.}\par}

\vskip 0.5cm

In this work, we investigate how the internal representations of Large Language Models (LLMs) change with inputs of increasing difficulty. Although sparsity changes may not be obvious for individual samples, dataset-level results reveal a consistent statistical trend: as task difficulty increases (e.g., harder questions, longer contexts, or more answer choices), the last hidden states of LLMs become systematically sparser. In short, \textbf{the harder the task, the sparser the representation}. This sparsity--difficulty relation is observable across diverse models and domains, suggesting that harder inputs drive more concentrated activation patterns in the last hidden state. Through a series of controlled analyses and learning-dynamics experiments, we show that representational density is a learned property of data familiarity: models develop rich, distributed representations for mastered patterns, while unfamiliar inputs default to sparser activations. Our finding is also actionable. We illustrate this with Sparsity-Guided Curriculum In-Context Learning (SG-ICL), which uses sparsity to select few-shot demonstrations matched to the query's difficulty, outperforming standard CoT and Auto-CoT baselines on MATH-500. Our study provides new insights into how LLM representations reflect task difficulty and how this signal can be leveraged for inference.
\par

\vskip 0.5cm

{\metadatafont
\begin{tabular}{@{}ll}
\textbf{Date:} & \paperdate \\
\textbf{Venue:} & Accepted at NeurIPS 2026 \\
\textbf{Code:} & \url{\codeurl} \\
\end{tabular}
}

}
\end{herobox}

\vspace{2.0em}

\section{Introduction}

Large Language Models (LLMs) serve as powerful tools for reasoning \citep{wei2022chain, li2025system, brown2020language} and knowledge-intensive interaction \citep{wang2024knowledge, chen2023felm, tongyideepresearchteam2025tongyideepresearchtechnicalreport}. Yet their reliability falters as the difficulty of the tasks increases \citep{lightman2024lets, li2025longcontext}, and the tasks we want LLMs to solve are only going to get harder. This makes it crucial to know what happens when a model faces difficulty: do those failures reflect a hard ceiling on what the model has learned, or merely a soft edge of it? External benchmarks cannot answer this. We need to look deeper inside the models' internal mechanisms. In this paper, we examine how an LLM's internal representation changes with task difficulty and what that reveals about when a model has truly learned a problem and when it has not.

We look at \textit{\textbf{sparsity}}: the extent to which a representation's activation is concentrated on a small subset of dimensions. Sparsity has long been studied as a sign of specialization in neural representations and neural networks \citep[{\it inter alia}]{olshausen1996emergence, glorot2011deep, frankle2018the, xiong-etal-2025-uncomp}, and shows up in LLM work on intrinsic dimensionality \citep{aghajanyanIntrinsicDimensionalityExplains2021, decaoSparseInterventionsLanguage2022}. But it has mostly been treated as a fixed property of the trained model \citep{liu2025teal}, rather than something that changes with the input. We find that it does change with the input. The effect is noisy on individual examples, but at the dataset level it is clear and consistent: \emph{as tasks get harder, the last hidden state gets sparser}.

\begin{figure}
    \centering\small
    \setlength{\tabcolsep}{0pt}
    \begin{tabular}{ccccc}
    &
    (\S\ref{sub:reasoning_complexity}) Reasoning Complexity &
    (\S\ref{sub:perturbation}) Answer Choice &
    (\S\ref{sub:knowledge_conflict}) Knowledge Conflict &
    (\S\ref{sub:long_context_reasoning}) Context Length \\
    \rotatebox{90}{\scriptsize \hspace{1em} Dense \(\xrightarrow{\hspace{0.5em}}\) Sparse} &
    \includegraphics[width=0.24\linewidth]{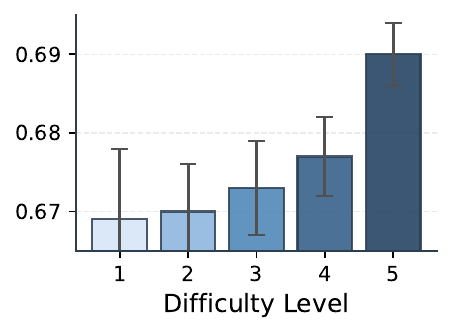} &
    \includegraphics[width=0.24\linewidth]{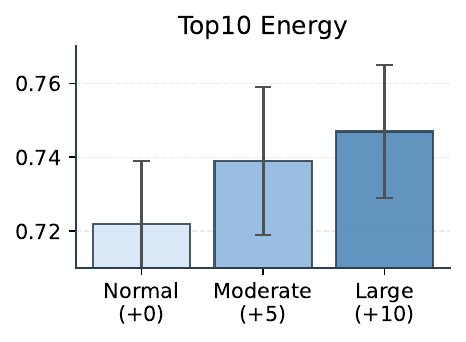} &
    \includegraphics[width=0.24\linewidth]{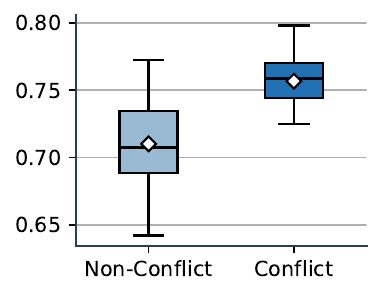} &
    \includegraphics[width=0.24\linewidth]{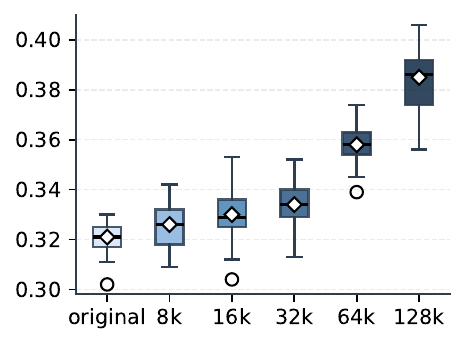} \\
    &
    \scriptsize\hspace{0.8em} Easy \(\xrightarrow{\hspace{2em}}\) Hard &
    \scriptsize\hspace{0.8em} Easy \(\xrightarrow{\hspace{2em}}\) Hard &
    \scriptsize\hspace{0.8em} Easy \(\xrightarrow{\hspace{2em}}\) Hard &
    \scriptsize\hspace{0.8em} Easy \(\xrightarrow{\hspace{2em}}\) Hard \\
    \end{tabular}
    \caption{\textbf{Harder Inputs Induce Sparser Representations.} Across all four controlled difficulty axes, the last hidden states become progressively sparser as tasks get harder. Results are shown for \exampleb{Qwen2.5-3B} using Top-10\% Energy; nevertheless, the same trend holds across difficulty datasets and settings, sparsity metrics, LLM family, and LLM sizes.}
    \label{fig:open_fig}
\end{figure}

Specifically, we measure sparsity in the last hidden state using Top-$k$ Energy for multiple values of $k$, along with Effective Rank and Hoyer. We study difficulty along four dimensions: reasoning complexity, number of answer choices, conflicting knowledge, and context length (Sections~\ref{sub:reasoning_complexity}--\ref{sub:long_context_reasoning}; Figure~\ref{fig:open_fig}). Across all four, the trend is the same: harder inputs yield sparser last-layer representations. The pattern holds across tasks, metrics, and model families, which makes it unlikely to be an artifact of any single task or model. In practical terms, increased sparsity means fewer dimensions carry most of the activation mass: the model uses a more concentrated set of features when the task gets harder.

Why does the representation become sparser as the task becomes harder? We find evidence that this regularity is a learned consequence of training rather than an architectural property. We study this both empirically and theoretically. Empirically, we train a toy Transformer from scratch on a synthetic knowledge graph: the pattern emerges only once the model has enough capacity and reaches convergence, although the training objective contains no explicit sparsity regularizer and provides no direct supervision for the harder-is-sparser pattern. Theoretically, since " difficulty " has many possible meanings, we use OOD rule compositions as a tractable proxy in our controlled setting. In a simplified cross-entropy + weight-decay model with kernelized representation dynamics, we prove a finite-horizon two-phase trajectory: concentration-based sparsity first increases during early feature selection, while density later recovers for inputs that receive sufficient training support. This provides a formal mechanism for our consolidation view: later in training, mastered inputs acquire dense, distributed representations, while inputs outside the learned regime remain comparatively sparse because this densification process is incomplete.

Finally, our finding of the sparsity–difficulty regularity is also actionable. We illustrate this with \emph{Sparsity-Guided Curriculum In-Context Learning} (SG-ICL), which uses sparsity to select few-shot demonstrations whose difficulty matches the query. On MATH-500, SG-ICL outperforms all major baselines, indicating that sparsity carries information that inference-time strategies can exploit.

In summary, we organize our study around three core research questions:
\begin{itemize}[leftmargin=*]
    \item \textbf{RQ1:} How does the geometry of the last hidden state evolve as tasks become increasingly difficult?
    \item \textbf{RQ2:} How does this sparsity–difficulty regularity arise during training?
    \item \textbf{RQ3:} How can this sparsity signal be practically leveraged to enhance model reasoning capabilities?
\end{itemize}

Guided by these questions, our primary findings are summarized as follows:

\paragraph{Sparsity increases with difficulty in a robust and controlled manner.} We uncover a consistent relationship where last hidden state activations become sparser as question difficulty increases, across multiple reasoning benchmarks including MATH-500 \citep{lightman2024lets}, DeepMATH-103K~\citep{he2026deepmathk}, LongReason \citep{ling2025longreason}, ConflictReason~\citep{wang2024resolving}, and additional settings reported in the appendix. To enable precise control over difficulty, we introduce \textbf{MMLU-Robust}, a benchmark that induces graded difficulty via controlled distractor augmentation.

\paragraph{Learning dynamics connect density to familiarity.}
We identify a fundamental link between representation density and data familiarity. By pretraining a small language model from scratch on a controlled synthetic knowledge-graph corpus and tracking hidden-state sparsity across checkpoints, we observe a two-phase learning dynamic: representations first become more concentrated during early feature selection, and then become denser as the model masters the training distribution. Equivalently, concentration-based metrics such as Top-$k$ Energy exhibit an inverted-U trajectory, while density-oriented metrics exhibit the corresponding recovery. This suggests that high activation density is a learned attribute of familiar inputs, while harder or less familiar inputs receive weaker consolidation and remain comparatively sparse. We further provide a \textbf{theoretical derivation} that explains this two-phase sparsity--density trajectory.

\section{Preliminaries}

\paragraph{Experimental Setup.} We conduct all experiments under the \emph{implicit reasoning setting}. Although Chain-of-Thought \citep[CoT;][]{wei2022chain} prompting also produces observable sparsity patterns, its effects are noticeably less stable. In contrast, the implicit reasoning setup yields clearer and more consistent activation dynamics, allowing us to present more interpretable visualizations. The prompt we use can be found in Appendix~\ref{apps:rp}. We focus primarily on the \emph{last hidden state} for the same reason. Although related sparsity trends are also present in intermediate layers, they are typically weaker and less consistent across tasks and models. By contrast, the final layer exhibits the clearest and most stable pattern, making it the most interpretable layer for our analysis. The details of Layer Selection can be found in Appendix~\ref{apps:layerproblem}. For model selection, we evaluate representative open-source LLMs, including Qwen2.5-1.5B/3B/7B/14B/32B/70B, Llama-3.2-1B/3B, and Llama-3.1-70B.\footnote{For brevity, we present Qwen2.5 results in the main text and defer results for other models to the appendix.}

\paragraph{Last Hidden State.}
We consider a transformer-based LM $f_{\theta}(\cdot)$ with $L$ stacked decoder layers, where each layer $\ell$ consists of a multi-head self-attention ($\mathrm{Attn}^{(\ell)}(\cdot)$) block and a feed-forward network ($\mathrm{MLP}^{(\ell)}(\cdot)$). For an input token sequence $x = (x_1, x_2, \ldots, x_T)$, we denote the hidden representation of token $x_t$ at layer $\ell$ as $h_t^{(\ell)} \in \mathbb{R}^d$, where $d$ is the hidden dimension. Each layer maintains a residual connection, and its output is computed as:
\begin{equation}
h_\text{mid} = h_t^{(\ell)} + \mathrm{Attn}^{(\ell)}\big(h_t^{(\ell)}\big); \quad
h_t^{(\ell+1)} = h_t^{(\ell)} + \mathrm{Attn}^{(\ell)}\big(h_t^{(\ell)}\big) + \mathrm{MLP}^{(\ell)}\big(h_\text{mid}\big).
\end{equation}
We focus on the output of each layer in the residual stream, which stores contextualized information aggregated across all tokens in the input prompt. Formally, we use the representation at the final layer and the final input token:
\begin{equation}
h_T^{(L)} = f_{\theta}^{(L)}(x_{1:T})[-1],
\end{equation}
which we refer to as the \emph{last hidden state}. This representation serves as the latent state used for next-token prediction and downstream tasks.

\paragraph{Sparsity.}
Prior work has shown that neural networks often exhibit concentration in their representations or effective parameterization. For example, fine-tuning can occur in low-dimensional subspaces~\citep{aghajanyanIntrinsicDimensionalityExplains2021}, motivating parameter-efficient tuning such as LoRA~\citep{huLoRALowRankAdaptation2022}. In this work, we study the \textit{representational sparsity} of the last hidden state $h=h_T^{(L)}\in\mathbb{R}^d$, defined as the extent to which a small subset of coordinates carries most of the activation mass.

Our primary sparsity metric is \textit{Top-$k$ Energy}. For a given percentage $k$, let $\mathcal{T}_k(h)$ denote the indices of the largest $k\%$ coordinates of $h$ in absolute value. We define
\begin{equation}
S_{\mathrm{Top}\text{-}k}(h)
=
\frac{
\sum_{i\in \mathcal{T}_k(h)} h_i^2
}{
\sum_{i=1}^{d} h_i^2
}.
\end{equation}
This metric measures how much of the total activation energy is captured by the largest coordinates. Since it is a ratio of squared energies, it is invariant to global rescaling of the hidden state. Higher Top-$k$ Energy indicates that activation mass is more concentrated in a small number of dimensions and therefore corresponds to higher sparsity. Additional metrics are reported in Appendix~\ref{apps:metrics_and_models}.

\paragraph{Difficulty}
Our object of study is \textbf{difficulty}: the property that makes some inputs harder for the model than others. Difficulty has no single rigorous formalization for LLMs: pretraining data is opaque, and ``harder'' can mean many things at once. Rather than commit to one definition, we study difficulty along four interpretable axes, each capturing a distinct form of stress on the model: (1)~\textbf{Reasoning complexity}: instances requiring deeper or multi-step inference (e.g., higher-difficulty math problems in MATH-500 and DeepMATH-103K); (2)~\textbf{Distractor interference}: multiple-choice items with a larger number of plausible distractors, expanding the solution space; (3)~\textbf{Knowledge conflict}: prompts whose context contradicts the model's parametric knowledge, forcing it to reconcile inconsistent signals; (4)~\textbf{Context length}: inputs with longer or more compositional contexts that demand longer-range integration. These four axes are not equivalent forms of the same construct; they each apply a different kind of pressure to the model. The fact that they all yield the same last-layer signature is part of what we find, not a framing assumption.

\section{RQ1: How does the geometry of the last hidden state evolve as tasks become increasingly difficult?}
\label{sec:rq1}

In this section, we present results addressing our first research question: \emph{How do last hidden states behave when models encounter hard questions?} We measure sparsity metrics across tasks with varying levels of difficulty along four dimensions: reasoning complexity (\S\ref{sub:reasoning_complexity}), answer choice expansion (\S\ref{sub:perturbation}), knowledge conflict (\S\ref{sub:knowledge_conflict}), and context length (\S\ref{sub:long_context_reasoning}). Across all four settings, we observe the same consistent pattern: the harder the task, the sparser the last hidden state.

\subsection{Sparsity \textit{vs} Reasoning Complexity}
We first examine reasoning complexity, asking whether the harder-is-sparser pattern appears consistently across both discretely annotated math difficulty levels and continuous difficulty scores.

MATH-500 \citep{hendrycks2021measuring}, a subset of the MATH benchmark with explicit difficulty annotations, serves as our first testbed. It contains 500 problems spanning five graded levels, from basic arithmetic (Level 1) to competition mathematics (Level 5). For each problem, we extract the last hidden state of the last token after applying a chat-template prompt to \exampleb{Qwen2.5-7B}, and compute Top-10\% Energy, Effective Rank, and Hoyer Sparsity. As shown in Figure~\ref{fig3}, sparsity tracks difficulty monotonically across all three metrics, with Pearson $|r| > 0.89$ on every panel.

To verify the pattern is not an artifact of the small benchmark's coarse 5-level annotation, we further evaluate on DeepMath-103K~\citep{he2026deepmathk}, a large-scale verifiable math reasoning dataset with continuous difficulty scores spanning roughly $[1, 10]$. We randomly sample 10{,}000 problems (seed 42) and evaluate \exampleb{Qwen2.5-7B}, extracting the last-layer hidden state of the last token after a chat-template prompt followed by ``Please give me the final answer.'' Figure~\ref{fig3b} plots three sparsity metrics (Top-5\% Energy, Top-10\% Energy, and Effective Rank) against the continuous difficulty score, with linear fits and 95\% confidence bands. Across all three metrics, both Pearson $r$ and Spearman $\rho$ have the expected signs (positive for Top-$k$ Energy and Hoyer Sparsity, negative for Effective Rank); all Pearson $p$-values fall below $10^{-119}$, and the same trend holds under non-parametric ranking. The dataset-level pattern is consistent with MATH-500: as the continuous difficulty score grows, the last hidden state becomes systematically sparser.

The same trend extends to other model families and sizes as shown in Appendix~\ref{apps:complex}.

\begin{figure}
    \centering
    \begin{subfigure}[b]{\linewidth}
        \centering
        \includegraphics[width=0.8\textwidth]{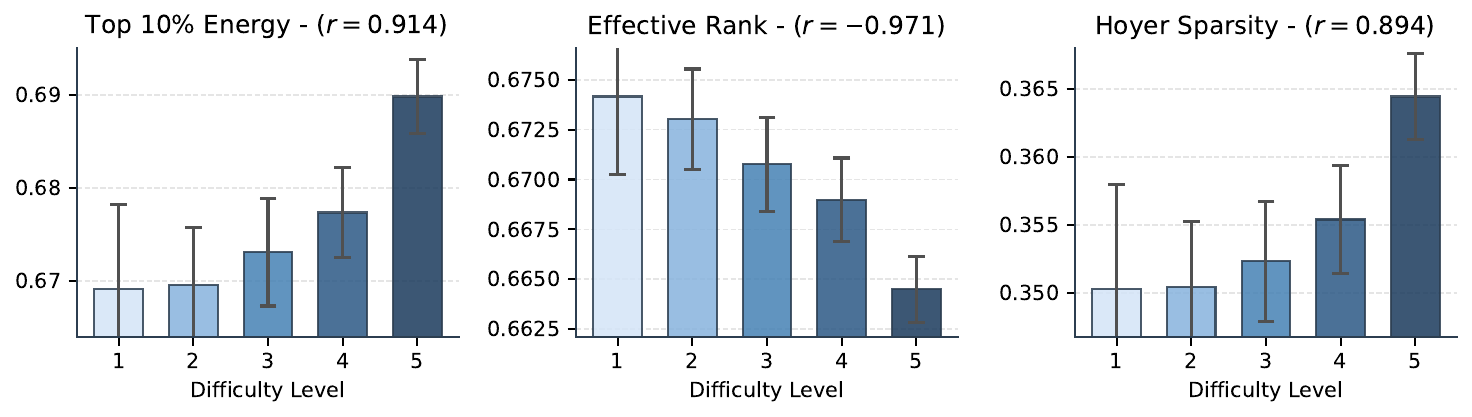}
        \caption{\textbf{MATH-500.} Sparsity strongly ($|r| > 0.89$) correlates with MATH-500 difficulty levels.}
        \label{fig3}
    \end{subfigure}

    \begin{subfigure}[b]{\linewidth}
        \centering
        \includegraphics[width=0.8\linewidth]{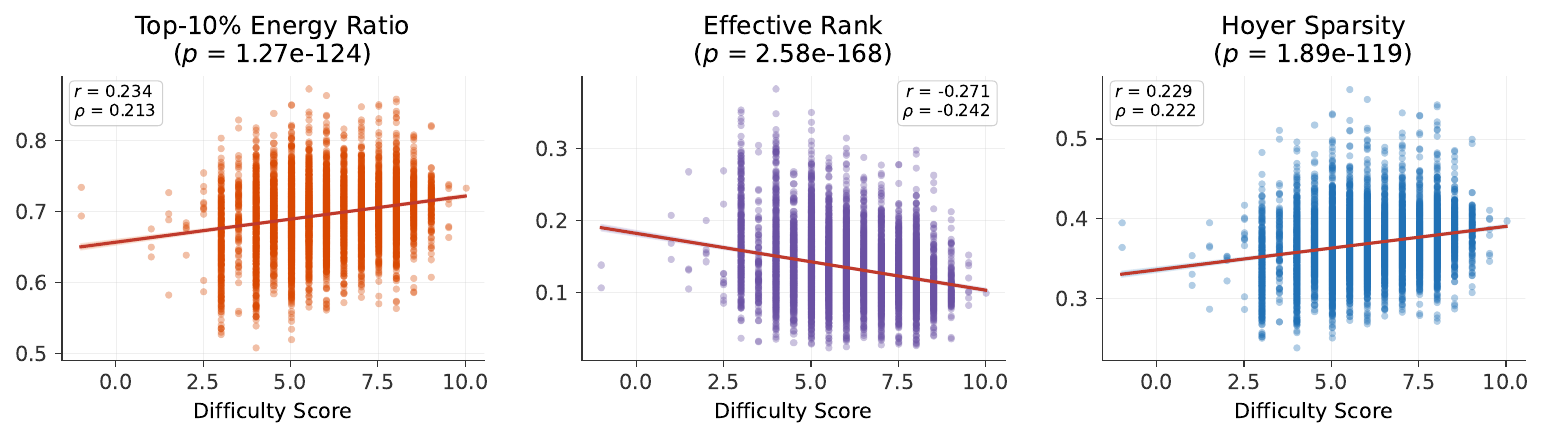}
        \caption{\textbf{DeepMATH-103K.} Sparsity across difficulty levels; all results significant at $p< 10^{-117}$.}
        \label{fig3b}
    \end{subfigure}
    \hfill 

    \caption{\textbf{Overview of Sparsity Analysis.} Both subfigures show that sparsity significantly correlates with difficulty for \exampleb{Qwen2.5-7B} across two datasets: MATH-500 (left) and DeepMATH-103K (right).}
    \label{fig:combined_figures}
\end{figure}

\label{sub:reasoning_complexity}

\subsection{Sparsity \textit{vs} Answer Choice Expansion}
\label{sub:perturbation}

In this setting, we add plausible but incorrect distractor options to multiple-choice questions presented to the LLM. This makes the task strictly harder \textit{ceteris paribus}, as the question content remains unchanged. This provides a controlled way to increase distractor interference while keeping the question stem unchanged. The setup remains simple and controlled, allowing us to study difficulty without confounding textual variation and to isolate the effect of difficulty.

\begin{figure}
    \centering
    \includegraphics[width=0.87\linewidth]{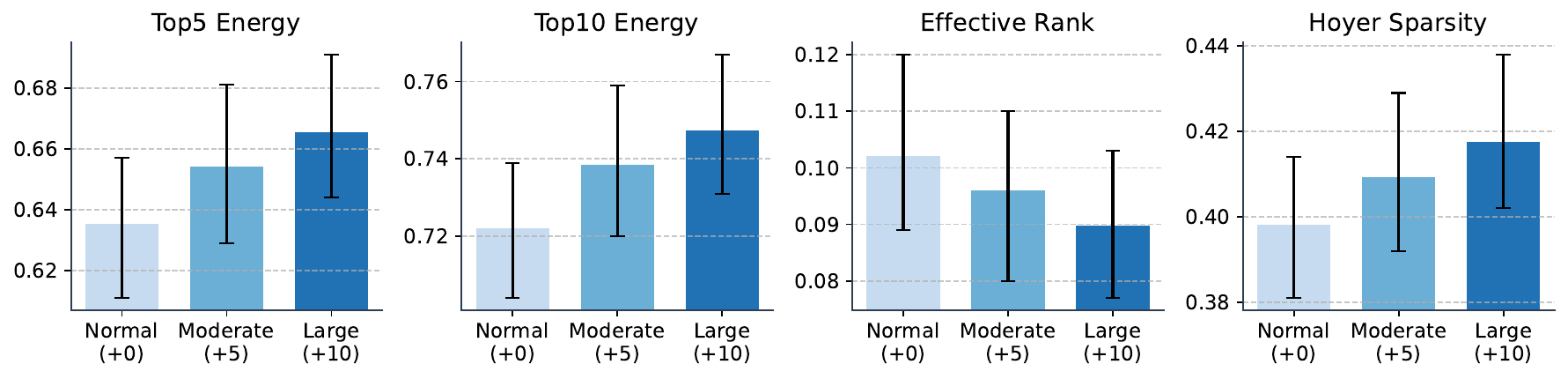}
    \caption{\textbf{Sparsity Metrics under Answer Choice Expansion.} Bar plots show mean sparsity across 14 disciplines for four metrics under Normal (+0), Moderate Expansion (+5), and Large Expansion (+10) on \exampleb{Qwen2.5-3B}. Error bars indicate the minimum and maximum across disciplines. Increasing task difficulty leads to higher sparsity.}
    \label{fig5}
    \vspace{-1em}
\end{figure}

Specifically, we use MMLU-Pro \citep{wang2024mmlu}, a multi-domain multiple-choice benchmark covering 14 academic disciplines. The original dataset provides ten answer options, which we use as the baseline \textbf{Normal} (+0) configuration. We introduce two harder configurations by adding plausible distractors: \textbf{Moderate Expansion} (+5), with five additional distractors (15 options total), and \textbf{Large Expansion} (+10), with ten additional distractors (20 options total). Question content is unchanged across configurations, isolating the effect of difficulty from textual variation. Together, these three configurations form \textbf{MMLU-Robust}; see Appendix~\ref{apps:mmlu} for details.

\autoref{fig5} presents the experimental results, confirming increased sparsity of final-layer hidden states when the tasks become more difficult. We analyzed four key sparsity metrics across all 14 academic areas in MMLU-Pro: Top-5\% Energy, Top-10\% Energy, Effective Rank, and Hoyer Sparsity. Across virtually all academic disciplines and for all four sparsity metrics, we observe a \textit{consistent, monotonic trend}: as the question difficulty increases by adding distractors (moving from $\text{Normal (+0)} \rightarrow \text{Moderate (+5)} \rightarrow \text{Large (+10)}$), the last hidden state exhibits higher sparsity. A detailed per-discipline breakdown, along with results for additional models, is provided in Appendix~\ref{app:mmlu_results}, confirming the same trends.

\subsection{Sparsity \textit{vs} Knowledge Conflict}
\label{sub:knowledge_conflict}

\begin{figure}
    \centering
    \scriptsize
    \setlength{\tabcolsep}{1pt}
    \renewcommand{\arraystretch}{0.98}
    \begin{tabular}{cccccc}
        \includegraphics[width=0.195\linewidth]{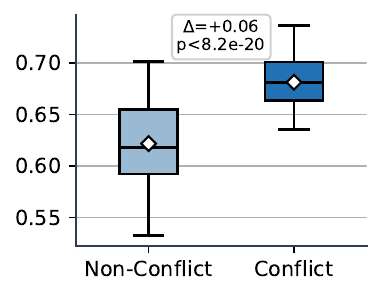} &
        \includegraphics[width=0.195\linewidth]{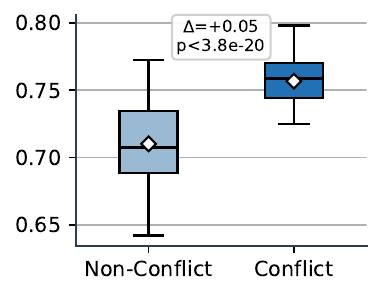} &
        \includegraphics[width=0.195\linewidth]{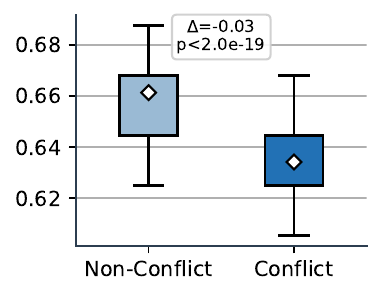} &
        \includegraphics[width=0.195\linewidth]{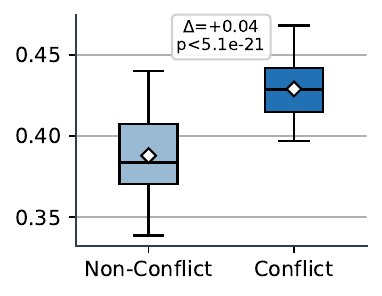} &
        \includegraphics[width=0.195\linewidth]{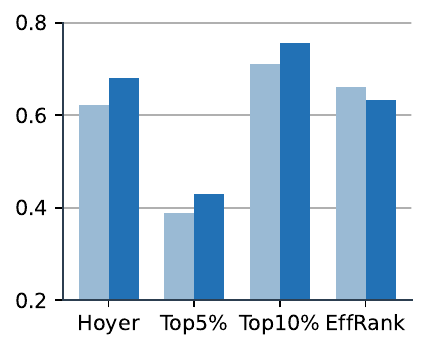} \\
        (a) Top-5\% &
        (b) Top-10\% &
        (c) Eff. Rank &
        (d) Hoyer &
        (e) Summary \\
    \end{tabular}
    \caption{\textbf{Sparsity Differences under Knowledge Conflict.}
    Across four sparsity metrics, the harder \textit{conflict}~(\cccbox{blue_c}) condition yields consistently sparser last-layer representations than the \textit{non-conflict}~(\cccbox{blue_nc}) condition on \exampleb{Qwen2.5-3B}. Arrows indicate the sparsity direction of each metric.}
    \label{fig6}
\end{figure}

Moreover, we use the~\cites {wang2024resolving} knowledge conflict dataset to further validate our hypothesis, which labels samples as conflict or non-conflict by pairing each question with either a truthful or a counterfactual context. In conflict scenarios, the provided context contradicts the model's parametric knowledge. Such conflicts force the model to reconcile external contextual evidence with its parametric knowledge, creating distributional stress.

As shown in ~\autoref{fig6}, conflict cases consistently exhibit higher sparsity across all four metrics. Defining $\Delta$ as the Conflict mean minus the Non-Conflict mean, the Top-5\% Energy (a), Top-10\% Energy (b), Effective Rank (c), and Hoyer Sparsity (d) changes are approximately $+0.06$, $+0.05$, $-0.03$, and $+0.04$, respectively. Higher Top-$k$ Energy and Hoyer Sparsity, together with lower Effective Rank, indicate more concentrated representations. These differences are \emph{statistically significant}, as confirmed by a paired \cites{studentProbableErrorMean1908} t-test, with all $p$-values $< 2.0 \times 10^{-19}$.  The Summary bar chart (e) synthesizes these findings, clearly illustrating that the latent last hidden representations for conflict instances are generally more compressed and concentrated on fewer dimensions. We provide additional analyses across more LLM families, more model sizes, and more knowledge conflict settings (e.g., the related \emph{distraction} phenomenon; \citet{Niu2025LlamaSL}) in Appendix~\ref{apps:kc}.

\subsection{Sparsity \textit{vs} Long Context Reasoning}
\label{sub:long_context_reasoning}

For the long-context reasoning setting, as shown in ~\autoref{fig2}, we use LongReason \citep{ling2025longreason}. This long context dataset offers controllable context lengths and incorporates diverse, realistic reasoning tasks. By analyzing the last hidden states across these controlled length variations, we discovered that the sparsity of the hidden states increases as the input text length increases, as \autoref{fig2}. This finding further supports our hypothesis that increasing task difficulty consistently drives the model towards a state of sharper activation compression in the last hidden state. In fact, the sparsity phenomenon induced by harder inputs is primarily a \emph{terminal behavior}. While intermediate representations remain relatively stable, the decisive shift in activation density between easier and harder inputs exclusively emerges in the final layer. We provide additional long-context analyses in Appendix~\ref{apps:long} and layer-level analyses in Appendix~\ref{apps:layerproblem} to discuss how the final layer shows the most pronounced and consistent sparsity trend across all the layers.
\begin{figure}
    \centering
    \includegraphics[width=0.93\textwidth]{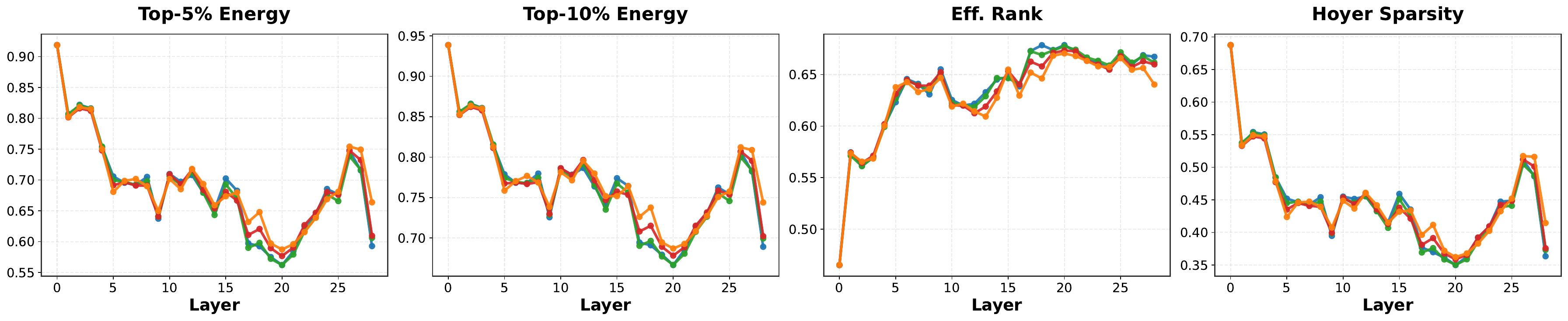}
    \parbox{\linewidth}{
    \centering
      \selectfont
      \textcolor{blue2}{\textbf{\texttt{\scriptsize
      CONTEXT: Sparsity metrics across all layers for Qwen2.5-7B under varying context lengths (8K $\rightarrow$ 64K).
      }}}%
    }

    \caption{\textbf{Layer-wise Sparsity across Context Lengths.}
    While intermediate layers show minimal variation across contexts, the final layers exhibit sharp divergence: longer contexts consistently produce sparser representations. The experiment is conducted on \textsc{LongReasonQA}~\citep{ling2025longreason}, where the background context length can be controlled. Line color indicates the input context length:
    \legmark{ctx8k}~8K, \legmark{ctx16k}~16K, \legmark{ctx32k}~32K, \legmark{ctx64k}~64K.}
    \label{fig2}
\end{figure}

\section{RQ2: How does this sparsity--difficulty regularity arise during training?}
\label{sec:rq2}

We now examine how the sparsity--difficulty connection develops during pretraining by training a toy-sized LM from scratch on synthetic data. In this section, we first describe the construction of the synthetic pre-training dataset (\S\ref{sub:controlled_env}), then analyze the pre-training results (\S\ref{sub:pretraining_result}), and finally connect these dynamics to the emergence of the sparsity--difficulty relationship through a combined empirical and theoretical account (\S\ref{sub:learning_dynamics}).

\subsection{Controlled Environment}
\label{sub:controlled_env}

We create a synthetic knowledge graph dataset that allows precise control over reasoning difficulty and distributional stress. The setting is inspired by \citet{wang2024grokking, wang2025do, yin2025are, ye2025how}, but we focus on sparsity–difficulty patterns in this controlled symbolic environment rather than the emergence of implicit reasoning.

We validate this mechanism by analyzing representations during pretraining. Even without task-specific fine-tuning, models pretrained solely on next-token prediction exhibit the same sparsity-difficulty pattern, confirming that this is a fundamental property of the language model.

Specifically, instead of relying on real-world knowledge graphs (where structural and semantic factors are entangled), we generate an artificial graph governed by a set of logical rules. Each rule defines how relations can be composed, e.g.,
\begin{equation}
    r_3(x, z) \Leftarrow r_1(x, y) \wedge r_2(y, z),
\end{equation}
ensuring that the resulting graph encodes multi-hop reasoning dependencies. Based on these rules, we construct the dataset in a three-stage process:

\paragraph{Preliminaries.}
We define a Knowledge Graph (KG) as $\mathcal{G} = (\mathcal{E}, \mathcal{R}, \mathcal{T})$, where $\mathcal{E}$ and $\mathcal{R}$ denote the sets of entities and relations, respectively. $\mathcal{T} = \{(h,r,t) \mid h,t \in \mathcal{E}, r \in \mathcal{R}\}$ represents the set of facts (triples). The problem difficulty is controlled by logical rules of varying lengths $L$.

\paragraph{Graph Construction.}
We construct the controlled knowledge graph $\mathcal{G}$ through a systematic generation pipeline. We first initialize the vocabulary of entities $\mathcal{E}=\{e_1, \ldots, e_{N}\}$ and relations $\mathcal{R} = \{P_1, \ldots, P_{N}\}$, along with a set of $K$ randomly sampled acyclic logic rules (lengths $L \in [L_{\min}, L_{\max}]$). The graph is seeded with random \textit{atomic triples} (explicit facts). Subsequently, we exhaustively apply the logic rules to these atomic triples to infer all possible \textit{deduced triples} (implicit facts). This process yields a dual-structure graph containing both direct observations and edges logically entailed by reasoning chains.

\paragraph{Sequence Serialization.}
To adapt the structured graph data for a language model, we linearize each triple $(h, r, t) \in \mathcal{T}$ into a token sequence (e.g., ``$e_1$ $P_5$ $e_2$''). The model is trained using the standard auto-regressive objective to predict the tail entity $t$ given the head $h$ and relation $r$.

\paragraph{Controlled Data Splitting.}
To distinguish between rote memorization and logical generalization, we employ a rigorous data splitting strategy. Specifically, we partition the triples $\mathcal{T}$ into a training set $\mathcal{D}_{\text{train}}$ and three distinct evaluation subsets: $\mathcal{D}_{\text{mem}}$ for assessing rote memorization, $\mathcal{D}_{\text{ID}}$ for testing in-distribution rule composition, and $\mathcal{D}_{\text{hard}}$ for probing out-of-distribution robustness.

Appendix~\ref{apps:pre} provides additional details on the synthetic data construction and the experimental setup used for pre-training the toy-sized LM.

\subsection{Pretraining Results}
\label{sub:pretraining_result}

\begin{table}[h]
    \caption{\textbf{Emergence of the Sparsity Phenomenon.}
    We report metrics across 3 levels. The gradient blue shading, observed uniquely in Trans-32-32, illustrates that the ``harder-is-sparser'' mechanism is a learned behavior.}

    \centering
    \resizebox{\linewidth}{!}{
        \renewcommand{\arraystretch}{1.2}
        \begin{tabular}{@{}l c c ccc ccc ccc ccc@{}}
        \toprule[1.5pt]
        \multirow{2.5}{*}{Model} &
        \multirow{2.5}{*}{Params} &
        \multirow{1.5}{*}{Training} &
        \multicolumn{3}{c}{Accuracy (\%, $\uparrow$)} &
        \multicolumn{3}{c}{Top-5 Energy (\%, $\uparrow$)} &
        \multicolumn{3}{c}{Top-10 Energy (\%, $\uparrow$)} &
        \multicolumn{3}{c}{Eff. Rank ($\downarrow$)} \\

        \cmidrule(lr){4-6}
        \cmidrule(lr){7-9}
        \cmidrule(lr){10-12}
        \cmidrule(lr){13-15}

        & & \multirow{-1.5}{*}{Steps} &
        $\mathcal{D}_{\text{mem}}$ & $\mathcal{D}_{\text{ID}}$ & $\mathcal{D}_{\text{OOD}}$ &
        $\mathcal{D}_{\text{mem}}$ & $\mathcal{D}_{\text{ID}}$ & $\mathcal{D}_{\text{OOD}}$ &
        $\mathcal{D}_{\text{mem}}$ & $\mathcal{D}_{\text{ID}}$ & $\mathcal{D}_{\text{OOD}}$ &
        $\mathcal{D}_{\text{mem}}$ & $\mathcal{D}_{\text{ID}}$ & $\mathcal{D}_{\text{OOD}}$ \\
        \midrule[1pt]
        Trans-8-8 & 0.02B & 2500 &
        64.00 & 42.00 & 31.00 &
        26.11 & 26.03 & 26.26 &
        42.12 & 42.00 & 42.34 &
        0.502 & 0.504 & 0.500 \\
        \midrule[0.5pt]
        Trans-16-8 & 0.04B & 2000 &
        71.00 & 50.00 & 24.00 &
        25.96 & 25.94 & 26.06 &
        41.98 & 41.98 & 42.15 &
        0.503 & 0.503 & 0.501 \\
        \midrule[0.5pt]
        \multirow{2}{*}{Trans-16-16} & \multirow{2}{*}{0.17B}
        & 2000 &
        85.00 & 56.00 & 44.00 &
        27.70 & 27.70 & 27.78 &
        43.70 & 43.66 & 43.76 &
        0.483 & 0.484 & 0.482 \\
        & & 2500 &
        81.00 & 58.00 & 37.00 &
        27.64 & 27.58 & 27.76 &
        43.56 & 43.52 & 43.69 &
        0.484 & 0.485 & 0.483 \\
        \midrule[0.5pt]
        \multirow{3}{*}{Trans-32-32} & \multirow{3}{*}{1.34B}
        & 1500 &
        \cellcolor{bg_c}78.00 & \cellcolor{bg_m}57.00 & \cellcolor{bg_h}37.00 &
        \cellcolor{bg_c}30.28 & \cellcolor{bg_m}30.54 & \cellcolor{bg_h}30.84 &
        \cellcolor{bg_c}46.44 & \cellcolor{bg_m}46.72 & \cellcolor{bg_h}47.04 &
        \cellcolor{bg_c}0.454 & \cellcolor{bg_m}0.451 & \cellcolor{bg_h}0.448 \\
        & & 2000 &
        \cellcolor{bg_c}83.00 & \cellcolor{bg_m}63.00 & \cellcolor{bg_h}34.00 &
        \cellcolor{bg_c}30.00 & \cellcolor{bg_m}30.23 & \cellcolor{bg_h}30.55 &
        \cellcolor{bg_c}46.11 & \cellcolor{bg_m}46.39 & \cellcolor{bg_h}46.70 &
        \cellcolor{bg_c}0.458 & \cellcolor{bg_m}0.455 & \cellcolor{bg_h}0.452 \\
        & & 2500 &
        \cellcolor{bg_c}81.00 & \cellcolor{bg_m}63.00 & \cellcolor{bg_h}38.00 &
        \cellcolor{bg_c}29.60 & \cellcolor{bg_m}29.75 & \cellcolor{bg_h}30.11 &
        \cellcolor{bg_c}45.67 & \cellcolor{bg_m}45.86 & \cellcolor{bg_h}46.22 &
        \cellcolor{bg_c}0.463 & \cellcolor{bg_m}0.461 & \cellcolor{bg_h}0.457 \\
        \bottomrule[1.5pt]
        \end{tabular}
    }
    \label{tab:maintable}
\end{table}

We use \textbf{Trans-$L$-$H$} to denote a Transformer with $L$ layers and $H$ attention heads (e.g., Trans-8-8 has 8 layers and 8 heads). We report Accuracy, Top-5/10 Energy, and Effective Rank across three difficulty levels:
Easy/Medium/Hard, which correspond to $\mathcal{D}_{\text{mem}}$, $\mathcal{D}_{\text{ID}}$, $\mathcal{D}_{\text{hard}}$ respectively. The findings highlight a clear correlation between model scale, training steps, and the emergence of sparsity. These results demonstrate that the relationship between task difficulty and sparsity is a learned phenomenon: it becomes consistent only once the model exits underfitting and reaches sufficient training convergence and capacity. The absence of a clear trend in smaller models may be attributed to the high data complexity relative to their limited capacity. We provide a detailed discussion on the graph data complexity in Appendix~\ref{apps:pre}.

We find that the previously reported ``harder-is-sparser'' effect
is not only present in instruction-tuned modern LLMs, but
also arises in standard pretrained Transformer models, where it
emerges as a pretraining-level property. In particular, this
pattern surfaces only once the model has sufficient capacity and reaches convergence (\autoref{tab:maintable}).

\subsection{Learning Dynamics: From Feature Selection to Consolidation}
\label{sub:learning_dynamics}
To study how sparsity evolves during training, we evaluate checkpoints every 200 steps over a 5,000-step training run and track the average sparsity of representations on a fixed set of training samples. Sparsity is measured by the fraction of activation energy captured by the Top-$k$ dimensions.

Nevertheless, \autoref{tab:maintable} and \autoref{fig8} reveal a counter-intuitive temporal evolution of sparsity. Focusing on the \textbf{Trans-32-32} model, we observe that sparsity does not simply increase throughout training. Instead, the Top-$5\%$ Energy first rises sharply during early training, increasing from roughly $0.275$ at initialization to a peak of about $0.318$ around 1300 steps, while accuracy simultaneously improves from near-random performance to above $60\%$. After this early phase, however, Top-$5\%$ Energy gradually decreases to around $0.278$ by 5000 steps, even as accuracy remains high and stable at approximately $75\%$. This suggests that, once the model has learned the core task structure, it no longer relies on a few sharply activated coordinates; instead, it develops a more distributed and robust representation for familiar inputs. A similar pattern appears for Top-$10\%$ Energy, which peaks around $0.483$ before falling back to roughly $0.442$ (with full curves reported in Appendix~\ref{app:learning_dynamics_details}).

As illustrated in \autoref{fig8}, the training process therefore exhibits a distinct two-phase dynamic.

\begin{figure}
    \centering
    \includegraphics[width=0.5\linewidth]{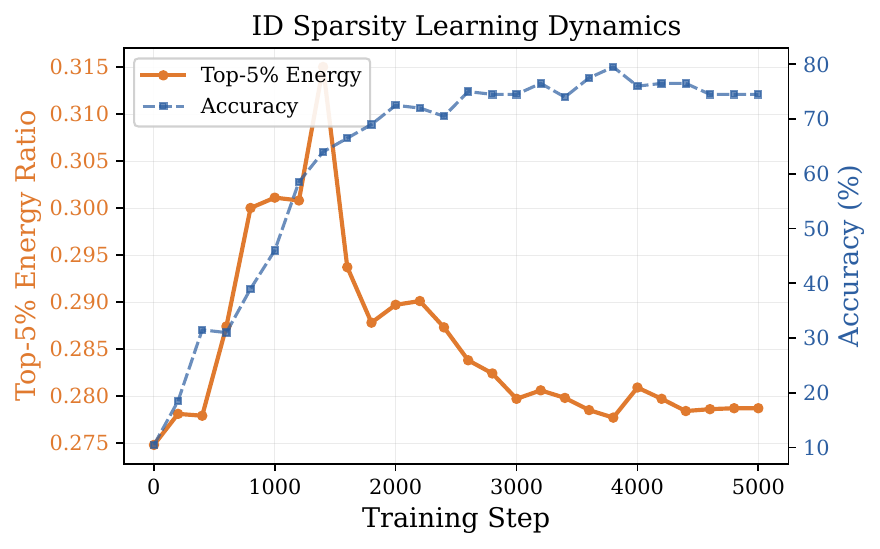}
    \caption{\textbf{Learning dynamics via Top-$k$ Energy.}
    Top-$k$ Energy first rises during early training, indicating feature selection and increasing activation concentration, and then gradually decreases while accuracy remains high, suggesting feature consolidation into a more distributed and robust representation.}
    \label{fig8}
\end{figure}

\paragraph{Phase I: Feature Selection.}
During early training, Top-$k$ Energy rises sharply, from about $0.275$ at initialization to a peak of roughly $0.315$ around 1.4K training steps. Over the same period, accuracy increases rapidly from near $10\%$ to above $60\%$. This indicates that the model increasingly concentrates activation mass on a small subset of dimensions while learning the core task structure, suggesting an early feature-selection phase in which task-relevant features are amplified and irrelevant or unstable directions are suppressed.

\paragraph{Phase II: Feature Consolidation.}
After this initial concentration phase, Top-$k$ Energy gradually decreases from its peak near $0.315$ to about $0.278$ by 5K steps, while accuracy remains high and stabilizes around $75$--$80\%$. This decrease should not be interpreted as a loss of learned structure; rather, it suggests \textit{representation maturation}. Once the core task structure has been learned, the model no longer relies on a few sharply activated coordinates but distributes information across a broader set of reliable features. This consolidation yields a smoother and more robust representation for familiar inputs, helping explain why in-distribution examples exhibit denser final-layer representations than harder OOD cases.

We give a theoretical justification for this two-phase sparsity--density dynamic using a simplified cross-entropy + weight-decay model with kernelized representation dynamics in Appendix~\ref{app:u_shape_theory}. We explain this as a shift from early weight-decay-driven feature contraction, which decreases normalized $\ell_1$, to later gradient-aligned updates on persistently hard examples that increase it.

\begin{figure*}[h]
\centering
\begin{tcolorbox}[
  width=\textwidth,
  colback=BoxLightBlue,
  colframe=BoxDarkBlue,
  coltitle=white,
  colbacktitle=BoxDarkBlue,
  fonttitle=\bfseries,
  title=Intuition: Learning Reduces Sparsity for Familiar Inputs
]
Learning can be viewed as a representation densification process, for a given input:
\[
\text{more training exposure}
\;\rightarrow\;
\text{better learned pattern}
\;\rightarrow\;
\text{lower sparsity}.
\]
Therefore, well-learned in-distribution inputs tend to produce dense final-layer representations.
For unseen or rarely observed inputs, this densification process is incomplete:
\[
\text{unseen / rarely seen pattern}
\;\rightarrow\;
\text{weakly learned pattern}
\;\rightarrow\;
\text{higher sparsity}.
\]
\end{tcolorbox}
\vspace{-0.5em}
\end{figure*}
Therefore, high activation density emerges as a \textit{learned privilege} of data familiarity. Since the language model learns rich distributed representations specifically for patterns it has mastered, it maintains high activation density in the last hidden state when facing in-distribution samples. When an input lies farther from the model's familiar distribution, the model is less able to recruit the broad distributed feature subspaces used for in-distribution reasoning. Consequently, activations concentrate on a narrower set of dimensions, leading to a sparser final-layer representation. This mechanism offers a natural explanation for our core observation: \emph{the farther the shift, the sparser the representation}. A fuller discussion, together with additional training curves across Top-$k$ Energy, Effective Rank, and accuracy on different model sizes, is provided in Appendix~\ref{app:learning_dynamics_details}.

\section{RQ3: How can this sparsity signal be practically leveraged to enhance model reasoning capabilities? }

Based on our insight that the last hidden representation sparsity serves as a reliable signal for task complexity, we propose a novel few-shot example selection strategy for LLM reasoning: Sparsity-Guided Curriculum In-Context Learning (SG-ICL). While standard in-context learning often selects demonstrations randomly or relies solely on semantic similarity, such as \citet{rubin2022learning, ma2023query} or through clustering-based methods such as Auto-CoT~\citep{zhang2023automatic}, it ignores the cognitive load required to process them. We argue that an effective prompt should act as a \textit{developmental curriculum}, guiding the model from rote, simple pattern matching to complex reasoning.

\begin{figure}
    \centering
    \includegraphics[width=0.4\linewidth]{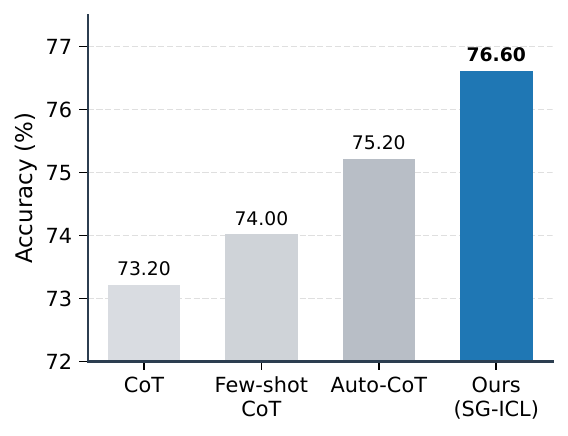}
    \caption{\textbf{Performance Comparison of Reasoning Strategies.} Our SG-ICL achieves an accuracy of 76.60\% on MATH-500 with \exampleb{Qwen2.5-7B}, substantially outperforming standard CoT baselines (zero-shot and few-shot) and the Auto-CoT baseline (75.20\%).}
    \label{fig:main_results_comparison}
\end{figure}

First, we assess the difficulty of all candidate examples in the demonstration pool. We compute the sparsity score $S(x)$ for each example $x$ using the Top-$k$ Energy of its last hidden state $\mathbf{h}_L$. We then sort the examples and group them into $K$ distinct difficulty levels (or bins) $\mathcal{B}_1, \dots, \mathcal{B}_K$:
\begin{equation}
\small
    S(x) = S_{\mathrm{Top}\text{-}k}(\mathbf{h}_L(x)), 
    \text{\scriptsize with } x \in \mathcal{B}_k \Leftrightarrow  \tau_{k-1} \le S(x) < \tau_k
\end{equation}
Here, $\tau$ represents the threshold boundaries derived from the percentiles of the sparsity distribution, allowing us to categorize examples from easy to hard systematically. During the inference phase for a specific test query $x_{query}$, our strategy employs a dual-criteria selection process:

    \paragraph{1. Semantic Filtering:} We first retrieve a candidate set $\mathcal{N}_{sem}$ consisting of the top-$N$ examples most semantically similar to $x_{query}$ by leveraging Sentence-BERT~\citep{reimers2019sentence} to compute dense vector embeddings.

    \paragraph{2. Difficulty Matching:} We then calculate the sparsity $S(x_{query})$ of the incoming query $x_{query}$ to identify its corresponding difficulty level $\mathcal{B}_{target}$. From $\mathcal{N}_{sem}$, we specifically select $k$ examples that not only maintain high semantic similarity but also align with the target difficulty level (i.e., falling within $\mathcal{B}_{target}$).

This process ensures the few-shot prompt is both contextually relevant and cognitively aligned with the query's content and complexity.

To validate this approach, we evaluate it on \exampleb{Qwen2.5-7B} using MATH-500. Empirical results on the MATH-500 dataset demonstrate that our Sparsity-Guided Curriculum strategy outperforms strong baselines, as illustrated in \autoref{fig:main_results_comparison}.

\section{Conclusion}
In this work, we identify a systematic connection between the internal representation geometry of LLMs and the difficulty of the tasks they face. Through analysis across diverse models, benchmarks, and controlled settings, we validate the phenomenon that \emph{the farther the shift, the sparser the representation.} Our findings show that this sparsification is not an incidental artifact of a single dataset or metric, but a robust representational pattern that appears most clearly in the final transformer layers. We further show that this pattern is tied to learning dynamics: as models acquire task structure, familiar inputs can recruit broader and denser feature subspaces, whereas harder or more distributionally shifted inputs activate a narrower set of dimensions.

\section*{Author Contributions}
Mingyu Jin conducted the experiments. Jingcheng Niu was responsible for scaling up the experiments to larger settings and polish the writing. Yutong Yin proposed the initial idea and developed the toy model and its theoretical analysis.

\section*{Acknowledgments}
We thank Minghao Guo, Xi Zhu, Bangwei Guo, Mengru Wang, Fei Sun, Xin Xu, Fengzhuo Zhang, Zirui Liu, Fan Yang, Wenyue Hua, and Jinman Zhao for their help and valuable discussions.

\clearpage

\bibliographystyle{acl_natbib}
\bibliography{main_cleaned}
\clearpage
\beginappendix
\clearpage

\section{Appendix}
\label{sec:appendix}


\subsection{Normalized \(\ell_1\) representation sparsity.}
We quantify hidden-state sparsity using the normalized \(\ell_1\) statistic. For a last hidden state \(h\in\mathbb{R}^d\), we define \(S_{L_1}(h)=\frac{1}{d}\sum_{j=1}^d |h_j|\), which measures the average activation magnitude across coordinates. Smaller values indicate that more coordinates are close to zero and that the representation is more concentrated in a subset of dimensions. This statistic is especially suitable for dataset-level OOD analysis because sparsity changes may be weak or noisy for individual samples, but become stable after averaging across examples. It also admits a smooth proxy \(N_\varepsilon(t)\), enabling the finite-horizon learning-dynamics analysis.

\subsection{A finite-horizon two-phase sparsity--density certificate for the normalized $\ell_1$ statistic}
\label{app:u_shape_theory}

In ~\autoref{fig8} and ~\autoref{fig:learning_dynamics_full}, we observe a robust two-phase sparsity--density trajectory during learning: the normalized $\ell_1$ magnitude of the last hidden state decreases early
(\emph{sparsification}), and later increases over another period (\emph{densification}). Equivalently, concentration-based metrics such as Top-$k$ Energy exhibit an inverted-U trajectory. This appendix provides a finite-horizon theoretical justification on a simplified learning task. Informally, the mechanism is: (i) early training is dominated by weight decay acting on a weak/near-zero head, which contracts feature magnitudes and thus decreases the normalized $\ell_1$ statistic; (ii) if at some later time the features are already sparse enough but a nontrivial subset of examples remains persistently hard, then the loss gradient aligns with the active coordinates and dominates the contraction over a finite window, forcing the statistic to increase.

\subsubsection{Simplified setting: data, model, objectives and gradients}
\label{app:u_shape_theory:setting}
We fix a training set $\{(x_i,y_i)\}_{i=1}^n$ with labels $y_i\in\{1,\dots,C\}$.
Let $\theta(t)$ denote the parameters that generate the last-layer representation, and define
\[
h_i(t) := h(x_i;\theta(t))\in\mathbb{R}^d,\quad i=1,\dots,n,
\]
(e.g., the last hidden state $h_T^{(L)}$ in the main text, evaluated on the final token of $x_i$).
Let $W(t)\in\mathbb{R}^{C\times d}$ denote the linear readout with rows $W_c(t)\in\mathbb{R}^d$.
We express the output logits $z_i(t)\in\mathbb{R}^C$ by
\[
z_i(t)  :=  W(t) h_i(t),\quad z_{i,c}(t)=\langle W_c(t), h_i(t)\rangle.
\]
Let $p_i(t)\in\Delta^{C-1}$ be the softmax probabilities:
\[
(p_i(t))_c := \frac{e^{z_{i,c}(t)}}{\sum_{c'=1}^C e^{z_{i,c'}(t)}},
\]
and let $e_{y_i}\in\mathbb{R}^C$ be the one-hot vector for label $y_i$.

\paragraph{Training objective (cross-entropy with parameter weight decay).} Define the per-example cross-entropy loss
\[
\ell_i(t) := \ell_{\mathrm{ce}} \bigl(z_i(t),y_i\bigr) = -\log\bigl((p_i(t))_{y_i}\bigr),
\]
and the dataset-averaged cross-entropy
\[
\mathcal{L}_{\mathrm{ce}}(\theta,W) := \frac{1}{n}\sum_{i=1}^n \ell_i(t).
\]
We regularize the \emph{parameters} that generate the representation (and optionally the readout) via $\ell_2$ weight decay:
\[
\mathcal{L}(\theta,W):=
\mathcal{L}_{\mathrm{ce}}(\theta,W)
+\frac{\lambda_h}{2}\|\theta\|_2^2
+\frac{\lambda_W}{2}\|W\|_F^2,
\]
where $\lambda_h\ge 0$ is the decay coefficient for $\theta$ and $\lambda_W\ge 0$ is optional readout decay.
In the analysis below, the effect of weight decay on $\theta$ is modeled at the feature level
as a contraction term $-\lambda_h h_i(t)$ (plus a residual), captured by the induced dynamics
in \eqref{eq:appC_induced_h_dynamics}.

\paragraph{Per-example representation gradient.}
Define the gradient of $\ell_i(t)$ with respect to the representation:
\[
g_i(t)  :=  \nabla_{h_i}\ell_i(t)
 =  W(t)^\top\bigl(p_i(t)-e_{y_i}\bigr)\in\mathbb{R}^d.
\]
Since $\|p_i(t)-e_{y_i}\|_2\le \sqrt{2}$, we have
\begin{equation}
\label{eq:appC_g_bound_basic}
\|g_i(t)\|_2 \le \sqrt{2} \|W(t)\|_{2\to 2},
\end{equation}
where $\|\cdot\|_{2\to 2}$ denotes the operator (spectral) norm.

\subsubsection{Induced (kernelized) representation dynamics with weight decay}
\label{app:u_shape_theory:induced_dyn}
Let $K(t)\in\mathbb{R}^{n\times n}$ be a symmetric matrix (a ``kernel'') with entries $K_{ik}(t)$.
We assume the induced dynamics of $h_i(t)=h(x_i;\theta(t))$ has the form
\begin{equation}
\label{eq:appC_induced_h_dynamics}
\dot h_i(t)
 =
-\frac{1}{n}\sum_{k=1}^n K_{ik}(t) g_k(t)
 - \lambda_h h_i(t)
 + r_i(t),
\end{equation}
where $\lambda_h\ge 0$ is the (effective) feature weight decay coefficient and $r_i(t)$ is a residual term
(e.g., due to non-frozen lower layers), assumed bounded on the horizons of interest.

\begin{remark}[A simple setting where~\eqref{eq:appC_induced_h_dynamics} holds exactly]
Consider a frozen backbone feature map $\phi:\mathcal{X}\to\mathbb{R}^m$ and a trainable linear
representation head $V(t)\in\mathbb{R}^{d\times m}$ such that
\begin{equation*}
    \begin{aligned}
        \theta(t) &\equiv V(t),\\
        h_i(t)&=h(x_i;\theta(t))=V(t)\phi(x_i)\eqqcolon V(t)\phi_i.
    \end{aligned}
\end{equation*}
Let the logits be $z_i(t)=W(t)h_i(t)$ and define $\ell_i(t)=\ell_{\mathrm{ce}}(z_i(t),y_i)$ as above.
If we train $V$ by gradient flow on $\mathcal{L}_{\mathrm{ce}}(V,W)$ with $\ell_2$ weight decay
$\frac{\lambda_h}{2}\|V\|_F^2$ (while $W$ may also be trained), then a direct chain-rule computation gives
\[\dot V(t)=-\frac1n\sum_{k=1}^n g_k(t) \phi_k^\top-\lambda_h V(t),\] which leads to
\begin{align}
\dot h_i(t)
&= \dot V(t)\phi_i \notag \\
&= -\frac{1}{n}\sum_{k=1}^n \langle \phi_i,\phi_k\rangle g_k(t)
-\lambda_h h_i(t). \notag
\end{align}
Therefore~\eqref{eq:appC_induced_h_dynamics} holds with the (time-independent) kernel
$K_{ik}=\langle \phi_i,\phi_k\rangle$ and residual $r_i(t)\equiv 0$.
\end{remark}

\subsubsection{Normalized $\ell_1$ sparsity and a smooth proxy}
\label{app:u_shape_theory:sl1}
Recall the normalized $\ell_1$ statistic (main text):
\[
S_{L_1}(h) = \frac{1}{d}\sum_{j=1}^{d}|h_j|.
\]
To obtain a differentiable drift identity, we introduce a smoothed proxy.

\paragraph{Smoothed primitives.}
Fix $\varepsilon>0$ and define, for $u\in\mathbb{R}$,
\begin{equation*}
\begin{aligned}
\rho_\varepsilon(u) &:= \sqrt{u^2+\varepsilon^2}, \\
\psi_\varepsilon(u) &:= \frac{u}{\sqrt{u^2+\varepsilon^2}}.
\end{aligned}
\end{equation*}
For $h\in\mathbb{R}^d$, define $\psi_\varepsilon(h)\in\mathbb{R}^d$ entrywise:
$(\psi_\varepsilon(h))_j := \psi_\varepsilon(h_j)$.

\paragraph{Dataset-averaged smoothed $\ell_1$ statistic.}
Define
\begin{equation}
\label{eq:appC_Neps_def}
N_\varepsilon(t) := \frac{1}{n}\sum_{i=1}^n \frac{1}{d}\sum_{j=1}^d \rho_\varepsilon\big(h_{i,j}(t)\big).
\end{equation}
Note that $N_\varepsilon(t)\to \frac{1}{n}\sum_{i=1}^n S_{L_1}(h_i(t))$ as $\varepsilon\to 0$.

\subsubsection{Drift identity for $N_\varepsilon$}
\label{app:u_shape_theory:drift}
Define the drift components
\begin{align}
\label{eq:appC_DCR_def}
D_\varepsilon(t)
&:= \frac{1}{n^2 d}\sum_{i=1}^n \sum_{k=1}^n K_{ik}(t) \big\langle \psi_\varepsilon(h_i(t)), g_k(t)\big\rangle,\\
C_\varepsilon(t)
&:= \frac{\lambda_h}{n d}\sum_{i=1}^n \sum_{j=1}^d
\frac{h_{i,j}(t)^2}{\sqrt{h_{i,j}(t)^2+\varepsilon^2}},\nonumber\\
R_\varepsilon(t)
&:= \frac{1}{n d}\sum_{i=1}^n \big\langle \psi_\varepsilon(h_i(t)), r_i(t)\big\rangle.\nonumber
\end{align}
\begin{lemma}[Exact drift identity]
\label{lem:appC_drift_identity}
For all $t$ for which the trajectories are differentiable,
\[
\dot N_\varepsilon(t) = -D_\varepsilon(t) - C_\varepsilon(t) + R_\varepsilon(t).
\]
\end{lemma}

\begin{proof}
Differentiate \eqref{eq:appC_Neps_def}:
\begin{equation*}
    \begin{aligned}
        \dot N_\varepsilon(t)
&= \frac{1}{n d}\sum_{i=1}^n \sum_{j=1}^d \psi_\varepsilon(h_{i,j}(t)) \dot h_{i,j}(t)\\
&= \frac{1}{n d}\sum_{i=1}^n \big\langle \psi_\varepsilon(h_i(t)), \dot h_i(t)\big\rangle.
    \end{aligned}
\end{equation*}
Substitute \eqref{eq:appC_induced_h_dynamics} and identify terms using \eqref{eq:appC_DCR_def}.
\end{proof}

\begin{lemma}[Two-sided decay bound]
\label{lem:appC_C_bound}
For all $t$,
\[
\lambda_h\big(N_\varepsilon(t)-\varepsilon\big)\le C_\varepsilon(t)\le \lambda_h N_\varepsilon(t).
\]
Consequently,
\begin{equation}
\label{eq:appC_comparison_ineq}
\begin{aligned}
&-D_\varepsilon(t)-\lambda_h N_\varepsilon(t)+R_\varepsilon(t)
\le \dot N_\varepsilon(t)\le -D_\varepsilon(t)-\lambda_h\big(N_\varepsilon(t)-\varepsilon\big)+R_\varepsilon(t).
\end{aligned}
\end{equation}
\end{lemma}

\begin{proof}
For any $u\in\mathbb{R}$,
\[
0\le \frac{u^2}{\sqrt{u^2+\varepsilon^2}} \le \sqrt{u^2+\varepsilon^2}=\rho_\varepsilon(u),
\]
\[
\frac{u^2}{\sqrt{u^2+\varepsilon^2}}
= \rho_\varepsilon(u)-\frac{\varepsilon^2}{\rho_\varepsilon(u)}
\ge \rho_\varepsilon(u)-\varepsilon,
\]
since $\rho_\varepsilon(u)\ge \varepsilon$.
Summing over $(i,j)$ and multiplying by $\lambda_h/(nd)$ yields the result, and then
\eqref{eq:appC_comparison_ineq} follows from \cref{lem:appC_drift_identity}.
\end{proof}

\subsubsection{Phase I: sparsification (early-time decrease trend + hitting a low level)}
\label{app:u_shape_theory:phase1}

\begin{assumption}[Phase I boundedness on a budget horizon]
\label{assump:appC_phase1}
Fix $T_{\mathrm I}>0$. Assume that on $t\in[0,T_{\mathrm I}]$:
\begin{enumerate}[label=(\roman*)]
\item Kernel bound: $\kappa(t):=\max_{i,k}|K_{ik}(t)|\le \kappa_0$.
\item Head bound: $\|W(t)\|_{2\to 2}\le M_0$.
\item Residual bound: $|R_\varepsilon(t)|\le r_0$.
\end{enumerate}
\end{assumption}

\begin{lemma}[Uniform bound on $|D_\varepsilon|$ on Phase I]
\label{lem:appC_D_abs_bound}
    Under \cref{assump:appC_phase1}, for all $t\in[0,T_{\mathrm I}]$, we have
\(
|D_\varepsilon(t)|\le \kappa_0 \sqrt{2} M_0/{\sqrt d}.
\)
\end{lemma}

\begin{proof}
By definition,
\[
|D_\varepsilon(t)|\le \frac{1}{n^2 d}\sum_{i,k}|K_{ik}(t)|\cdot\big|\langle \psi_\varepsilon(h_i(t)), g_k(t)\rangle\big|.
\]
Since $|\psi_\varepsilon(u)|\le 1$, $\|\psi_\varepsilon(h_i)\|_2\le \sqrt d$, hence
$\big|\langle \psi_\varepsilon(h_i),g_k\rangle\big|\le \sqrt d \|g_k\|_2$.
Using \eqref{eq:appC_g_bound_basic}, $\|g_k\|_2\le \sqrt{2}\|W\|_{2\to 2}\le \sqrt{2}M_0$,
and $|K_{ik}(t)|\le \kappa_0$. Summing over $i,k$ yields
\begin{align}
|D_\varepsilon(t)|
&\leq \frac{1}{n^2 d}\cdot n^2 \cdot \kappa_0 \cdot \sqrt d \cdot \sqrt{2}M_0 \notag \\
&= \kappa_0 \frac{\sqrt{2}M_0}{\sqrt d}. \notag
\end{align}
\end{proof}

\begin{lemma}[Phase I dynamic]
\label{lem:appC_phase1}
Under \cref{assump:appC_phase1}, define
\(
B_0 := \kappa_0 \sqrt2 M_0/{\sqrt d} + r_0.
\)
Then for all $t\in[0,T_{\mathrm I}]$, we have
\begin{equation}
\label{eq:appC_phase1_diff_ineq}
\dot N_\varepsilon(t)\le-\lambda_h\big(N_\varepsilon(t)-\varepsilon\big) + B_0,
\end{equation}
and therefore
\begin{equation}
\label{eq:appC_phase1_solution}
N_\varepsilon(t)-\varepsilon
\le
\big(N_\varepsilon(0)-\varepsilon\big)e^{-\lambda_h t}
+\frac{B_0}{\lambda_h}\big(1-e^{-\lambda_h t}\big).
\end{equation}
\end{lemma}

\begin{proof}
From the upper comparison inequality \eqref{eq:appC_comparison_ineq},
\[
\dot N_\varepsilon(t)\le -D_\varepsilon(t)-\lambda_h\big(N_\varepsilon(t)-\varepsilon\big)+R_\varepsilon(t).
\]
Use $-D_\varepsilon(t)\le |D_\varepsilon(t)|$, \cref{lem:appC_D_abs_bound}, and $|R_\varepsilon(t)|\le r_0$
to obtain \eqref{eq:appC_phase1_diff_ineq}. Solving the scalar comparison ODE yields
\eqref{eq:appC_phase1_solution} and the hitting-time claim.
\end{proof}
\begin{corollary}[Phase I decrease trend + certified hitting time]
From Lemma \ref{lem:appC_phase1}, we directly have that $\dot N_\varepsilon(t)<0$ whenever $N_\varepsilon(t)>\varepsilon+\frac{B_0}{\lambda_h}$.
Moreover, for any target level $L>\varepsilon+\frac{B_0}{\lambda_h}$, define the (comparison) hitting time
\[
t_L := \frac{1}{\lambda_h}\log\left(
\frac{N_\varepsilon(0)-\varepsilon-\frac{B_0}{\lambda_h}}{L-\varepsilon-\frac{B_0}{\lambda_h}}
\right),
\]
(where the logarithm is well-defined when $N_\varepsilon(0)>\varepsilon+\frac{B_0}{\lambda_h}$).
If $t_L\le T_{\mathrm I}$, then $N_\varepsilon(t_L)\le L$.

\end{corollary}

\begin{remark}[What Phase I does and does not claim]
\label{rem:appC_phase1_scope}
Phase I certifies (i) an early-time decrease trend whenever $N_\varepsilon$ is above
$\varepsilon+\frac{B_0}{\lambda_h}$, and (ii) the existence of a time $t_L\le T_{\mathrm I}$ at which
$N_\varepsilon(t_L)$ reaches any prescribed level $L>\varepsilon+\frac{B_0}{\lambda_h}$.
It does not claim monotone decrease beyond the certified horizon.
\end{remark}

\subsubsection{Phase II: densification on a prescribed finite horizon}
\label{app:u_shape_theory:phase2}

\begin{assumption}[Phase II one-shot densification assumption]
\label{assump:appC_phase2}
There is a start time $\tau > 0$ and horizon length $H>0$ that make the following conditions hold for every $t\in[\tau,\tau+H]$.

\medskip
\noindent\textbf{(A) Boundedness and residual controls.}
There exist constants $\kappa_0,M,r_0,r_h\ge 0$ such that
\(
\kappa(t):=\max_{i,k\in[n]}|K_{ik}(t)|\le \kappa_0\), \(\|W(t)\|_{2\to 2}\le M\),
\(
|R_\varepsilon(t)|\le r_0\), \(\max_{i\in[n]}\|r_i(t)\|_2\le r_h.
\)

\medskip
\noindent\textbf{(B) Hard aligned subset with weak cross-talk.}
There exist a subset $S\subset[n]$ with $|S|\ge \rho n$ and constants
$k_{\min}>0$, $\delta\in[0,1)$, $\alpha_0>0$, $m_0>0$, $a_0>0$
such that for all $i\in S$ and all $t\in[\tau,\tau+H]$, we have
\(
K_{ii}(t)\ge k_{\min}\), \(\sum_{k\neq i}|K_{ik}(t)|\le \delta K_{ii}(t)
\), and
\(
\alpha_i(t):=1-(p_i(t))_{y_i}\ge \alpha_0
\).
Letting $c_i(t)\in\arg\max_{c\neq y_i} z_{i,c}(t)$ (a highest-scoring wrong class), we have
\[
z_{i,c_i(t)}(t)-\max_{c\neq y_ic\neq c_i(t)}z_{i,c}(t)\ge m_0,
\]
and with $v_i(t):=W_{y_i}(t)-W_{c_i(t)}(t)$,
\[
\langle \psi_\varepsilon(h_i(t)), v_i(t)\rangle \ge a_0\sqrt d.
\]

\medskip
\noindent\textbf{(C) Easy complement + weak coupling to hard gradients.}
There exist constants $\bar\alpha\in(0,1)$ and $\bar\delta\ge 0$ such that for all $t\in[\tau,\tau+H]$, we have
\[
\alpha_k(t):=1-(p_k(t))_{y_k}\le \bar\alpha, \quad \forall k\notin S,\]
\[\sum_{k\in S}\sum_{i\notin S}|K_{ik}(t)| \le \bar\delta |S| k_{\min}.
\]

\medskip
\noindent\textbf{(D) Feasibility (signal dominates leakage and residual).}
Define
\(
\zeta:=(C-2)e^{-m_0}\),
\(b_0:=(\alpha_0-\zeta)a_0-2\zeta M
\), \(
\mu_D:=\rho k_{\min}\big(b_0-\delta\sqrt2 M\big)/{n\sqrt d}
\), \(\gamma_{\mathrm{easy}}
:=\kappa_0\sqrt2 M\Big(\sqrt{\bar\alpha}+\bar\delta\Big)/{\sqrt d}\) and \(\mu_D^{\mathrm{eff}}:=\mu_D-\gamma_{\mathrm{easy}}\).
Assume
\(
b_0>\delta\sqrt2 M\) and \( \mu_D^{\mathrm{eff}}>r_0
\).

\medskip
\noindent\textbf{(E) Initial margin below the Phase II threshold for the full horizon.}
Let
\(
N_*:=(\mu_D^{\mathrm{eff}}-r_0) /{\lambda_h}
\).
Define
\(
D_{\max}:=\kappa_0 \sqrt2 M/{\sqrt d}\), \(
H_0:=\max_{i\in[n]}\|h_i(\tau)\|_2
\) and
\begin{equation*}
\begin{aligned}
H_{\max}
&:= e^{\lambda_h H}H_0
+\frac{\kappa_0\sqrt{2}M+r_h}{\lambda_h}
\big(e^{\lambda_h H}-1\big), \\
N_{\max}
&:= \frac{H_{\max}}{\sqrt d}+\varepsilon, \\
V
&:= D_{\max}+\lambda_h N_{\max}+r_0.
\end{aligned}
\end{equation*}
Assume there exists $\eta>0$ such that
\begin{equation}
\label{eq:appC_phase2_margin}
N_\varepsilon(\tau)\le N_* - VH - \eta.
\end{equation}
\end{assumption}

\begin{lemma}[Top-2 dominance from runner-up separation]
\label{lem:appC_top2}
Assume \cref{assump:appC_phase2}(B) holds.
Then for any $i\in S$ and any $t\in[\tau,\tau+H]$,
\[
\sum_{c\neq y_i,\ c\neq c_i(t)} (p_i(t))_c \le (C-2)e^{-m_0} =: \zeta,\]
\[
(p_i(t))_{c_i(t)} \ge \alpha_i(t)-\zeta.
\]
\end{lemma}

\begin{proof}
Fix $i\in S$ and $t\in[\tau,\tau+H]$. For any $c\neq y_i,c_i(t)$,
\begin{equation*}
\begin{aligned}
\frac{(p_i(t))_c}{(p_i(t))_{c_i(t)}}
&= \exp\big(z_{i,c}(t)-z_{i,c_i(t)}(t)\big) \\
&\leq e^{-m_0}.
\end{aligned}
\end{equation*}
Summing over $c\neq y_i,c_i(t)$ yields
\begin{equation*}
    \begin{aligned}
        \sum_{c\neq y_i,c\neq c_i(t)}(p_i(t))_c &\le (C-2)e^{-m_0}(p_i(t))_{c_i(t)}\\
        &\le (C-2)e^{-m_0}\\&=\zeta.
    \end{aligned}
\end{equation*}
Since $\alpha_i(t)=\sum_{c\neq y_i}(p_i(t))_c=(p_i(t))_{c_i(t)}+\sum_{c\neq y_i,c\neq c_i(t)}(p_i(t))_c$,
we obtain $(p_i(t))_{c_i(t)}\ge \alpha_i(t)-\zeta$.
\end{proof}

\begin{lemma}[Diagonal negativity on $S$]
\label{lem:appC_diag_neg}
Assume \cref{assump:appC_phase2}(A)--(D) holds. Then for all $i\in S$ and $t\in[\tau,\tau+H]$,
\[
\big\langle \psi_\varepsilon(h_i(t)), g_i(t)\big\rangle \le -b_0\sqrt d.
\]
\end{lemma}

\begin{proof}
Fix $i\in S$ and $t\in[\tau,\tau+H]$. Recall
\begin{equation*}
    \begin{aligned}
        g_i(t)&=W(t)^\top(p_i(t)-e_{y_i})\\
&=-\sum_{c\neq y_i}(p_i(t))_c (W_{y_i}(t)-W_c(t)).
    \end{aligned}
\end{equation*}
Hence
\begin{equation*}
    \begin{aligned}
        \langle \psi_\varepsilon(h_i),g_i\rangle
&=-(p_i)_{c_i}\langle \psi_\varepsilon(h_i),v_i\rangle\\&-\sum_{c\neq y_i, c\neq c_i}(p_i)_c \langle \psi_\varepsilon(h_i),W_{y_i}-W_c\rangle.
    \end{aligned}
\end{equation*}
By alignment, $\langle \psi_\varepsilon(h_i),v_i\rangle\ge a_0\sqrt d$.
Also $\|\psi_\varepsilon(h_i)\|_2\le \sqrt d$ and
$\|W_{y_i}-W_c\|_2\le 2\|W\|_{2\to2}\le 2M$, so
$|\langle \psi_\varepsilon(h_i),W_{y_i}-W_c\rangle|\le 2M\sqrt d$.
Using \cref{lem:appC_top2} and $\alpha_i(t)\ge \alpha_0$ gives
$(p_i)_{c_i}\ge \alpha_0-\zeta$ and $\sum_{c\neq y_i,c\neq c_i}(p_i)_c\le \zeta$.
Therefore
\begin{equation*}
    \begin{aligned}
        \langle \psi_\varepsilon(h_i),g_i\rangle
&\le -(\alpha_0-\zeta)a_0\sqrt d + (2M\sqrt d)\zeta\\
&= -\big((\alpha_0-\zeta)a_0-2\zeta M\big)\sqrt d
\\ &= -b_0\sqrt d.
    \end{aligned}
\end{equation*}
\end{proof}

\begin{lemma}[Uniform negativity of $D_\varepsilon$ on the Phase II window (with easy-complement control)]
\label{lem:D-negative-full}
Assume Assumption~\ref{assump:appC_phase2}(B)--(E) hold.
Then for all $t\in[\tau,\tau+H]$, we have
\(
D_\varepsilon(t)\le -\mu_D^{\mathrm{eff}}.
\)
\end{lemma}
\begin{proof}
Fix $t\in[\tau,\tau+H]$. Decompose
\begin{equation*}
    \begin{aligned}
        D_\varepsilon(t)&=\frac{1}{n^2d}\sum_{i\in S}\sum_{k=1}^n K_{ik}(t) \langle \psi_\varepsilon(h_i(t)),g_k(t)\rangle\\ &+
\frac{1}{n^2d}\sum_{i\notin S}\sum_{k=1}^n K_{ik}(t) \langle \psi_\varepsilon(h_i(t)),g_k(t)\rangle\\
&=:D_S(t)+D_{S^c}(t).
    \end{aligned}
\end{equation*}

\paragraph{Step 1: bound $D_S(t)$ by a negative constant.}
As in the original proof, for $i\in S$ we have
(Lemma~\ref{lem:appC_diag_neg})
\begin{equation*}
\langle \psi_\varepsilon(h_i(t)),g_i(t)\rangle \leq -b_0\sqrt d,
\end{equation*}
and for all $i,k$,
\begin{equation*}
\begin{aligned}
\big|\langle \psi_\varepsilon(h_i(t)),g_k(t)\rangle\big| &\leq \|\psi_\varepsilon(h_i(t))\|_2\|g_k(t)\|_2 \\ &\leq \sqrt d\cdot \sqrt{2}M.
\end{aligned}
\end{equation*}
Combining this bound with the weak cross-talk condition
$\sum_{k\neq i}|K_{ik}|\leq \delta K_{ii}$, together with
$K_{ii}\geq k_{\min}$ and $|S|\geq \rho n$, we obtain
\begin{equation*}
\begin{aligned}
D_S(t)&\leq -\frac{|S|}{n^2d}k_{\min}\big(b_0-\delta\sqrt{2}M\big)\sqrt d \\ &\leq -\frac{\rho k_{\min}}{n\sqrt d}\big(b_0-\delta\sqrt{2}M\big)\\ &=\frac{\rho k_{\min}\big(\delta\sqrt{2}M-b_0\big)}{n\sqrt d}\\ &=-\mu_D.
\end{aligned}
\end{equation*}

\paragraph{Step 2: bound the complement term $D_{S^c}(t)$ using (C).}
Split $D_{S^c}(t)$ by $k\in S$ and $k\notin S$:
\begin{equation*}
    \begin{aligned}
        D_{S^c}(t)&=\frac{1}{n^2d}\sum_{i\notin S}\sum_{k\in S}K_{ik}\langle\psi(h_i),g_k\rangle\\&+\frac{1}{n^2d}\sum_{i\notin S}\sum_{k\notin S}K_{ik}\langle\psi(h_i),g_k\rangle\\
&=:T_{\mathrm{hard}}(t)+T_{\mathrm{easy}}(t).
    \end{aligned}
\end{equation*}
\emph{Easy-gradient part.}
For $k\notin S$, (C) gives $\alpha_k(t)\le\bar\alpha$, hence
\begin{equation*}
    \begin{aligned}
        \|g_k(t)\|_2& \le \|W(t)\|_{2\to2}\|p_k(t)-e_{y_k}\|_2 \\
        & \le M\sqrt{2\alpha_k(t)}\\
        & \le \sqrt2 M \sqrt{\bar\alpha}.
    \end{aligned}
\end{equation*}
Therefore $\big|\langle\psi(h_i),g_k\rangle\big|\le \|\psi(h_i)\|_2\|g_k\|_2\le \sqrt d\cdot \sqrt2 M\sqrt{\bar\alpha}$,
so with $|K_{ik}|\le\kappa_0$,
\begin{equation*}
    \begin{aligned}
        |T_{\mathrm{easy}}(t)|
&\le \frac{1}{n^2d}\sum_{i\notin S}\sum_{k\notin S}|K_{ik}|\cdot \sqrt d(\sqrt2 M\sqrt{\bar\alpha})\\
&\le \frac{1}{n^2d}\cdot n^2\cdot \kappa_0\cdot \sqrt d(\sqrt2 M\sqrt{\bar\alpha})
\\ &=\frac{\kappa_0\sqrt2 M}{\sqrt d}\sqrt{\bar\alpha}.
    \end{aligned}
\end{equation*}

\emph{Hard-gradient coupling from $i\notin S$ to $k\in S$.}
For $k\in S$, we use the crude bound $\|g_k(t)\|_2\le \sqrt2 M$, hence
$\big|\langle\psi(h_i),g_k\rangle\big|\le \sqrt d\cdot \sqrt2 M$.
Thus
\begin{equation*}
    \begin{aligned}
        |T_{\mathrm{hard}}(t)|
&\le \frac{1}{n^2d}\sum_{i\notin S}\sum_{k\in S}|K_{ik}|\cdot \sqrt d(\sqrt2 M)\\
&=\frac{\sqrt2 M}{n^2\sqrt d}\sum_{k\in S}\sum_{i\notin S}|K_{ik}|.
    \end{aligned}
\end{equation*}
By (C), $\sum_{k\in S}\sum_{i\notin S}|K_{ik}|\le \bar\delta |S| k_{\min}\le \bar\delta (\rho n) k_{\min}$,
so
\begin{equation*}
    \begin{aligned}
        |T_{\mathrm{hard}}(t)|&\le \frac{\sqrt2 M}{n^2\sqrt d}\cdot \bar\delta (\rho n) k_{\min}
\\ &=\frac{\rho k_{\min}}{n\sqrt d} \bar\delta \sqrt2 M\\
&\le \frac{\kappa_0\sqrt2 M}{\sqrt d} \bar\delta,
    \end{aligned}
\end{equation*}
where in the last step we used $\rho k_{\min}/n \le \kappa_0$ (since $k_{\min}\le \max_{i,k}|K_{ik}|\le \kappa_0$ and $\rho\le 1$).
Therefore
\begin{equation*}
    \begin{aligned}
        |D_{S^c}(t)|&\le |T_{\mathrm{hard}}(t)|+|T_{\mathrm{easy}}(t)|\\
&\le \frac{\kappa_0\sqrt2 M}{\sqrt d}\Big(\bar\delta+\sqrt{\bar\alpha}\Big)
\\ &= \gamma_{\mathrm{easy}}.
    \end{aligned}
\end{equation*}

\paragraph{Step 3: combine.}
We have $D_\varepsilon(t)=D_S(t)+D_{S^c}(t)\le -\mu_D+|D_{S^c}(t)|
\le -\mu_D+\gamma_{\mathrm{easy}}= -\mu_D^{\mathrm{eff}}$.

\end{proof}

\begin{lemma}[Finite-velocity bound]
\label{lem:appC_velocity}
Assume \cref{assump:appC_phase2}(A) holds. With $V$ as in \cref{assump:appC_phase2}(E),
\[
|\dot N_\varepsilon(t)|\le V,\quad \forall t\in[\tau,\tau+H].
\]
\end{lemma}

\begin{proof}
From \cref{lem:appC_drift_identity},
\[
|\dot N_\varepsilon|\le |D_\varepsilon|+|C_\varepsilon|+|R_\varepsilon|.
\]
Under (A), applying \cref{lem:appC_D_abs_bound} with $(\kappa_0,M)$ gives $|D_\varepsilon(t)|\le D_{\max}$.
Also \cref{lem:appC_C_bound} implies $0\le C_\varepsilon(t)\le \lambda_h N_\varepsilon(t)$.
It remains to bound $N_\varepsilon(t)$ on $[\tau,\tau+H]$. From \eqref{eq:appC_induced_h_dynamics}, we have
\begin{equation*}
    \begin{aligned}
        \|\dot h_i(t)\|_2
&\le \frac{1}{n}\sum_k |K_{ik}(t)| \|g_k(t)\|_2 + \lambda_h\|h_i(t)\|_2 + \|r_i(t)\|_2\\
&\le \kappa_0\sqrt2 M + \lambda_h\|h_i(t)\|_2 + r_h,
    \end{aligned}
\end{equation*}
where we used $\frac1n\sum_k |K_{ik}|\le \max_{k}|K_{ik}|\le \kappa_0$ and \eqref{eq:appC_g_bound_basic}.
Let $u_i(t):=\|h_i(t)\|_2$. Then
$$u_i'(t)\le \kappa_0\sqrt2 M+\lambda_h u_i(t)+r_h.$$
So comparison yields $u_i(t)\le H_{\max}$ for all $t\in[\tau,\tau+H]$.
Using $\rho_\varepsilon(u)\le |u|+\varepsilon$ and $\|h_i\|_1\le \sqrt d\|h_i\|_2$ gives
$N_\varepsilon(t)\le N_{\max}$, hence $|C_\varepsilon(t)|\le \lambda_h N_{\max}$.
Finally, (A) gives $|R_\varepsilon(t)|\le r_0$.
Combining: $|\dot N_\varepsilon(t)|\le D_{\max}+\lambda_h N_{\max}+r_0 = V$.
\end{proof}

\begin{theorem}[Phase II densification on a prescribed horizon]
\label{thm:appC_phase2}
Under \cref{assump:appC_phase2}, $N_\varepsilon(t)$ is strictly increasing on $[\tau,\tau+H]$ and
\[
\dot N_\varepsilon(t) \ge \lambda_h\eta,
\quad \forall t\in[\tau,\tau+H].
\]
\end{theorem}

\begin{proof}
By \cref{lem:D-negative-full}, $D_\varepsilon(t)\le -\mu_D^{\mathrm{eff}}$ for all $t\in[\tau,\tau+H]$.
Also \cref{lem:appC_C_bound} gives $C_\varepsilon(t)\le \lambda_h N_\varepsilon(t)$, and (B) gives $R_\varepsilon(t)\ge -r_0$.
Thus, for all $t\in[\tau,\tau+H]$,
\begin{equation*}
    \begin{aligned}
        \dot N_\varepsilon(t)
&= -D_\varepsilon(t)-C_\varepsilon(t)+R_\varepsilon(t)\\
&\ge \mu_D^{\mathrm{eff}}-\lambda_h N_\varepsilon(t)-r_0
\\ &= \lambda_h\big(N_* - N_\varepsilon(t)\big).
    \end{aligned}
\end{equation*}
By \cref{lem:appC_velocity}, $| \dot N_\varepsilon(t) |\le V$, hence for any $t\in[\tau,\tau+H]$,
\begin{equation*}
\begin{aligned}
N_\varepsilon(t)
&\leq N_\varepsilon(\tau)+V(t-\tau) \\
&\leq N_\varepsilon(\tau)+VH \\
&\leq N_*-\eta.
\end{aligned}
\end{equation*}
using \eqref{eq:appC_phase2_margin}. Therefore $N_* - N_\varepsilon(t)\ge \eta$ and so
\begin{equation*}
\begin{aligned}
N_\varepsilon(t)
&\leq N_\varepsilon(\tau)+V(t-\tau) \\
&\leq N_\varepsilon(\tau)+VH \\
&\leq N_*-\eta.
\end{aligned}
\end{equation*}
This implies strict increase on the whole window.
\end{proof}

\subsubsection{Finite-horizon two-phase trajectory: Phase I decrease trend + bridge to Phase II}
\label{app:u_shape_theory:ushape}
\begin{theorem}[Finite-horizon sparsify-then-densify trigger]
\label{thm:appC_sparsify_then_densify}
Assume Phase I boundedness (\cref{assump:appC_phase1}) holds on $[0,T_{\mathrm I}]$ and that $\lambda_h>0$.
Let $B_0$ be as in \cref{lem:appC_phase1}.

\medskip\noindent
\textbf{(1) Phase I: early decrease and certified entry into a low-$\ell_1$ regime.}
For any level $L>\varepsilon+\frac{B_0}{\lambda_h}$, define $t_L$ as in \cref{lem:appC_phase1}.
If $t_L\le T_{\mathrm I}$, then:
\begin{enumerate}[label=(\roman*)]
\item $N_\varepsilon(t)$ is strictly decreasing whenever $N_\varepsilon(t)>\varepsilon+\frac{B_0}{\lambda_h}$;
\item $N_\varepsilon(t_L)\le L$ (so the trajectory reaches the prescribed low level by time $t_L$).
\end{enumerate}

\medskip\noindent
\textbf{(2) Phase II trigger: densification on any later valid window.}
Fix any horizon length $H>0$.
Suppose that at some (a priori unknown) later time $\tau\ge 0$ the Phase II assumptions
(\cref{assump:appC_phase2}) hold on the window $[\tau,\tau+H]$.
Then on that window, $N_\varepsilon$ is strictly increasing with the uniform slope lower bound
\[
\dot N_\varepsilon(t)\ge \lambda_h\eta,\quad \forall t\in[\tau,\tau+H],
\]
where $\eta$ is the margin parameter appearing in \cref{assump:appC_phase2}.

\end{theorem}
\begin{proof}
Part (1) is exactly \cref{lem:appC_phase1}.
Part (2) is exactly \cref{thm:appC_phase2} applied to the (hypothetical) window $[\tau,\tau+H]$.
\end{proof}
\begin{remark}[How to read the certificate]
This theorem does \emph{not} claim that a single continuous interval exhibits a full U-shape without further information.
Instead, it provides two independent, checkable guarantees along a training trajectory:
(i) an \emph{early-time sparsification certificate} that forces $N_\varepsilon$ to decrease and reach any prescribed
reasonable level $L>\varepsilon+\frac{B_0}{\lambda_h}$ within the Phase I validity horizon, and
(ii) a \emph{densification trigger}: whenever (at any later time) the trajectory enters a finite window on which
the persistent-hardness and alignment conditions of \cref{assump:appC_phase2} hold, $N_\varepsilon$ must increase
throughout that window at rate at least $\lambda_h\eta$.
In this sense, a U-shape pattern arises whenever the dynamics first satisfies the Phase I conditions and later
\emph{activates} a Phase II window.
\end{remark}

\subsection{Extended Related Work}
\subsubsection{Sparsity in Deep Neural Network}
Sparsity in deep learning is traditionally studied through two distinct lenses: computational efficiency and representational disentanglement~\citep{hoefler2021sparsity, cheng2024survey}.

From an efficiency perspective, weight sparsity aims to reduce model size and inference cost. Pioneering work in network pruning demonstrated that significant portions of parameters can be removed without performance degradation~\citep{han2015learning}. This line of inquiry culminated in the \textit{Lottery Ticket Hypothesis}, which posits that dense networks contain sparse, trainable subnetworks capable of matching the original model's accuracy~\citep{frankle2018the}. This hypothesis was subsequently extended to Transformer architectures, particularly BERT. Research indicates that pre-trained BERT models contain matching subnetworks at extreme sparsity levels (e.g., 40\%--90\%) that perform comparably on downstream tasks~\citep{chen2020lottery, prasanna2020bert, gordon2020compressing}. Beyond imposed sparsity, recent empirical studies reveal that LLMs exhibit a high degree of \textit{intrinsic activation sparsity}. Despite being trained as dense networks, the activations within Transformer Feed-Forward Networks (FFNs) are highly sparse, with only a small fraction of neurons firing for any given input token. This behavior, often termed the ``Lazy Neuron'' phenomenon, becomes more pronounced as model scale increases~\citep{li2023the}.

Further analysis characterizes this as \textit{contextual sparsity}: specific inputs activate predictable, sparse sub-graphs of the network, suggesting that LLMs implicitly learn modular structures without explicit architectural constraints~\citep{liu2023deja}. ~\citet{zhao2025on, zhao2025geometry, zhao2024implicit} view language modeling as a classification problem of "predicting the distribution of the next word for each context," the sparsity pattern of the language data strongly determines the geometry of the trained representation. They found that Next Token Prediction (NTP) training implicitly favors a "sparse + low-rank" structure in the logit space.

\subsubsection{LLM Curriculum Reasoning}
Curriculum Learning (CL), introduced by \citet{Bengio2009CurriculumL}, is a training strategy that exposes models to examples in an easy-to-hard progression, rather than in random order~\citep{wang2021survey}. Early NLP applications likewise leveraged CL principles. In unsupervised grammar induction, a dependency parser was trained on short, simple sentences first and incrementally included longer, more complex sentences, yielding improved parsing accuracy~\citep{spitkovsky2010baby}. With the advent of deep learning, Curriculum Learning (CL) was widely adopted in Neural Machine Translation (NMT), where examples were typically sorted by sentence length or word rarity to accelerate convergence \citep{platanios2019competence, xu-etal-2020-curriculum}. In the era of Pre-trained Language Models (PLMs) such as BERT, CL strategies shifted towards optimizing data scheduling to improve sample efficiency during the massive pre-training phase~\citep{nagatsuka-etal-2021-pre, lee2022efficient}.

Most recently, the rise of LLMs has reinvigorated CL, particularly in enhancing complex reasoning capabilities and instruction following \citet{wang2025dump}. Unlike traditional methods relying on superficial metrics (e.g., length), modern CL for LLMs focuses on \textit{semantic complexity} and \textit{reasoning depth}. For instance, \citet{xu2024wizardlm} proposed \textit{Evol-Instruct}, a method that incrementally rewrites instructions to increase difficulty, effectively creating a curriculum for instruction tuning. Similarly, \citet{mukherjee2023orca} demonstrated that learning from "explanation traces" in a progressive manner allows smaller models to imitate the reasoning processes of larger foundation models. Furthermore, in mathematical and logical reasoning tasks, CL has been utilized to transition models from simple single-step problems to multi-step reasoning chains, significantly mitigating the difficulty of solving complex problems directly \citep{luo2025wizardmath}.

\subsubsection{Interpretability in LLM Reasoning}

The interpretation of reasoning in Large Language Models (LLMs) has evolved from analyzing surface-level generations to probing the internal causal mechanisms that drive them~\citep{zhao2024explainability, zhao2024towards, wang2024knowledge}. A natural starting point for interpretability is to treat generated explanations, most prominently CoT, as a proxy for the model's decision-making process~\citep{wei2022chain}. The structure of these demonstrations also matters: changing reasoning-step length can affect accuracy even when their information content is preserved~\citep{jin2024impact}. However, recent scholarship challenges this assumption, characterizing CoT as often unfaithful. Evidence suggests that intermediate reasoning steps can be post-hoc rationalizations that do not causally determine the final prediction~\citep{turpin2023language}, and increasing reasoning burden does not necessarily equate to computational transparency~\citep{lanham2023measuring, arcuschin2025chainofthought}. Output confidence is also an imperfect proxy: \citet{zeng2025uncertainty} show that adjusting first-token probabilities can distort uncertainty estimates for multiple-choice questions while leaving answer accuracy unchanged. Consequently, the field has shifted focus from verifying \textit{plausible explanations} to establishing \textit{causal accounts} of model behavior. \citet{jin2025disentangling} explicitly separate memory recall from reasoning during inference, helping identify the sources of model errors.

To bridge the gap between model inputs and outputs, mechanistic interpretability develops tools that aim to isolate components that are necessary and sufficient for particular behaviors, and to characterize how information is represented and transformed across the network~\citep{wangInterpretabilityWildCircuit2022, gantla2025exploring}. A prominent line of work uses causal interventions: most notably, causal tracing and activation patching to perform controlled swaps between activations induced by “clean” versus “corrupted” inputs, thereby localizing where task-relevant signals reside and empirically testing how those signals propagate through layers and attention pathways~\citep{wangInterpretabilityWildCircuit2022, zhao2025beyond, meng2022locating}. Complementary approaches focus less on direct causal disruption and more on interpretability via readouts: representation-decoding frameworks such as Patchscopes map hidden states into structured textual probes, offering a unified interface for inspecting intermediate computations under a variety of intervention patterns~\citep{ghandeharioun2024patchscopes}. Related “lens” methods, including the Logit Lens, project intermediate activations into the vocabulary (logit) space to track how candidate outputs become linearly recoverable across depth; while lightweight and often used diagnostically, these readouts are frequently most informative when paired with causal tests that distinguish genuine computation from superficial decodability~\citep{nostalgebraist2020logit, wendler2024llamas, geva2021transformer, belrose2023eliciting}. Layer-wise probing further suggests that more complex concepts become accessible at deeper layers~\citep{jin2025exploring}.

At a finer level of analysis, mechanistic interpretability work tries to break down a model’s behavior into concrete computational pathways (“circuits”) and smaller features that can be followed as they appear and change across different inputs~\citep{elhage2021framework, elhage2022toy}. A central thread in this direction explains in-context learning through specific attention-head motifs: most notably induction heads that appear to implement a concrete pattern-copying mechanism~\citep{olsson2022context}. To move beyond correlational stories, subsequent work has introduced more stringent validation procedures, such as causal scrubbing and automated circuit discovery (ACDC), which use behavior-preserving resampling and controlled interventions to test whether a proposed circuit is genuinely explanatory rather than an artifact of spurious co-activation~\citep{conmy2023towards}.

In parallel, feature-centric approaches address the pervasive polysemanticity of dense activations by learning sparse decompositions: Sparse Autoencoders (SAEs) attempt to “unsuperpose” representations into more interpretable directions that are often closer to monosemantic features, enabling a complementary level of analysis that is not limited to predefined motifs like individual heads~\citep{huben2024sparse, bricken2023monosemanticity}. Building on SAE features, SAGE iteratively tests and refines feature explanations using activation feedback~\citep{han2026sage}, while SAE-SSV learns supervised steering vectors within task-relevant sparse latent subspaces~\citep{he2025sae}. Overall, the literature increasingly reflects a convergence in methodology, pairing explanatory narratives with rigorous causal tests and feature-level decomposition in order to debug and interpret reasoning processes in modern LLMs.

\subsubsection{Learning Dynamics}

A complementary line of work studies how optimization shapes representations over the course of training. Classical results on the \emph{implicit bias} of gradient-based optimization show that, even without explicit regularization, gradient descent on separable classification problems converges in direction to max-margin solutions, linking training dynamics to margin growth and norm evolution~\citep{soudry2018implicit, Ji2019TheIB}; related analyses extend these behaviors to stochastic gradient descent~\citep{nacson2019sgd}. More recently, phase-transition-like learning phenomena such as \emph{grokking} highlight that models can move from memorization to generalization late in training, motivating mechanistic accounts in terms of regime changes in training dynamics and feature learning~\citep{kumar2024grokking, demoss2025complexity}.

Recent work has also begun to analyze learning dynamics in the context of LLM adaptation. \citet{ren2025learningdynamics} propose a learning-dynamics framework for \emph{LLM fine-tuning}, including supervised fine-tuning and preference optimization, characterizing how learning on specific examples influences predictions on others and using this lens to explain several counter-intuitive fine-tuning behaviors. These perspectives suggest that simple representation-level statistics, such as norms, energy concentration, and effective dimensionality, can serve as useful probes of how features are selected, amplified, or suppressed during learning.

Our work is complementary to this literature. Rather than analyzing the parameter trajectory induced by optimization or fine-tuning, we study an inference-time representational signature associated with increasing OOD difficulty. In particular, we show that harder inputs consistently induce more concentrated last-layer representations, and that this ``harder-is-sparser'' pattern emerges only under sufficiently capable models or training regimes in our controlled setting. Thus, our results connect learning dynamics to OOD generalization from a representation-level perspective: as models acquire more structured reasoning behavior, difficulty is reflected not only in output accuracy but also in the geometry and sparsity of their hidden states.

\subsection{All Metrics}
\label{apps:metrics_and_models}
\subsubsection{Sparsity Metrics}
We analyze the sparsity of the activation vector $\mathbf{h} \in \mathbb{R}^{d}$ extracted from the last hidden state of the model. Since raw activations are rarely absolute zeros in floating-point representations, we utilize these metrics that capture the \textit{effective sparsity} (i.e., the peakedness of the distribution).

\paragraph{Hoyer Sparsity.}
Derived from the relationship between the $\ell_1$ and $\ell_2$ Norms, Hoyer sparsity measures how close a vector is to being sparse (containing mostly zeros). For an activation vector $\mathbf{h}$ of dimension $d$:
\begin{equation}
    \mathrm{Hoyer}(\mathbf{h}) = \frac{\sqrt{d} - \|\mathbf{h}\|_1/\|\mathbf{h}\|_2}{\sqrt{d} - 1},
\end{equation}
where $\|\cdot\|_1$ is the $\ell_1$ norm and $\|\cdot\|_2$ is the Euclidean norm.  A value of \textbf{0} indicates a dense, uniform vector ($h_i = c, \forall i$) and a value of \textbf{1} indicates maximum sparsity (only one component is non-zero).

\paragraph{Gini Index.}
Originally used in economics to measure income inequality, the Gini index here quantifies the inequality of neural activation strengths. High inequality implies that a few neurons dominate the representation while the majority remain suppressed (sparse).
\begin{equation}
    \mathrm{Gini}(\mathbf h)=\frac{\sum_{i=1}^{d}(2i-d-1) |h|_{(i)}}{d\sum_{i=1}^{d}|h|_{(i)}},\quad |h|_{(1)}\le\cdots\le |h|_{(d)},
\end{equation}
where $|h_{(i)}|$ are the sorted absolute values of the activations. A higher Gini index correlates with a sparser representation.

\paragraph{Effective Rank.}
Effective rank quantifies how concentrated the activation energy is across dimensions.
Given $\mathbf{h}\in\mathbb{R}^d$, define a distribution over dimensions by $p_i=\frac{h_i^2}{\sum_{j=1}^{d} h_j^2}$, and entropy $H(\mathbf{p})=-\sum_{i=1}^{d} p_i\log(p_i+\epsilon)$. We then compute
\begin{equation}
    \mathrm{EffRank}(\mathbf{h})=\frac{\exp(H(\mathbf{p}))}{d}\in(0,1].
\end{equation}
Larger values indicate more evenly spread energy (denser representations), while smaller values indicate stronger
concentration on a few dimensions (sparser representations).

As a complementary metric, we also compute the normalized $\ell_1$ magnitude:
\begin{equation}
\textstyle
S_{L_1}(h)
=
\frac{1}{d}
\sum_{i=1}^{d}
|h_i|.
\end{equation}
Unlike Top-$k$ Energy, $S_{L_1}$ depends on the overall scale of the hidden state. We therefore use it as a complementary metric, comparing representations within the same model, layer, token position, and evaluation protocol, where hidden-state norms are broadly comparable. Under this controlled setting, a lower $S_{L_1}$ indicates that activation mass is concentrated on fewer coordinates, providing additional evidence of increased sparsity.
\subsubsection{Model Parameter Metrics}
We validate our findings across a range of transformer model sizes, ranging from 0.3M to 1.3B parameters, using the Llama architecture, we need to know how to compute the model Parameter size  $N_{params}$. As detailed in ~\autoref{tab:maintable}, the model complexity is controlled by the hidden size ($d_{model}$), the intermediate MLP size ($d_{mlp}$), the number of attention heads ($H$), and the network depth ($L$). Assuming a standard Llama-based architecture (utilizing SwiGLU activation and Rotary Embeddings), the approximate parameter count $N_{params}$ for the non-embedding layers is calculated as:
\begin{equation}
    N_{params} \approx L \cdot \left( \underbrace{4 d_{model}^2}_{\text{Attention}} + \underbrace{3 d_{model} \cdot d_{mlp}}_{\text{FeedForward (SwiGLU)}} \right)
\end{equation}
where the attention block consists of $W_Q$, $W_K$, $W_V$, $W_O$, contributing $4 d_{model}^2$. The FeedForward block (SwiGLU) involves three projections (gate, up, down), contributing $3 d_{model} \cdot d_{mlp}$.

\begin{figure}[h]
    \centering
    \begin{tikzpicture}
        \begin{axis}[
            width=0.55\linewidth,   
            height=6cm,
            xlabel={\textbf{Model Size}},
            ylabel={\textbf{Count (\#)}},
            title={\textbf{Model Depth \& Attention Heads Scaling}}, 
            xtick={1,2,3,4,5,6,7,8,9,10,11,12,13}, 
            xticklabels={0.3M, 0.7M, 1.3M, 2.6M, 5.3M, 10.5M, 21M, 42M, 83.9M, 167.8M, 335.6M, 671.2M, 1342.4M},
            xticklabel style={rotate=45, anchor=east, font=\footnotesize},
            ymin=0, ymax=35,
            ylabel style={font=\bfseries},
            grid=major,
            grid style={dashed, gray!30},
            mark size=2.5pt, 
            line width=1.2pt,
            legend pos=north west, 
            legend style={font=\small, fill=white, draw=gray!20, fill opacity=0.9}
        ]

            \addplot[color=DeepRed, mark=triangle*, mark options={fill=DeepRed}] coordinates {
                (1, 2) (2, 4) (3, 2) (4, 4) (5, 8) (6, 4) (7, 8) (8, 16)
                (9, 8) (10, 16) (11, 32) (12, 16) (13, 32)
            };
            \addlegendentry{\# Layers (Depth)}

            \addplot[color=GoldDetail, mark=diamond*, mark options={fill=GoldDetail}] coordinates {
                (1, 2) (2, 2) (3, 4) (4, 4) (5, 4) (6, 8) (7, 8) (8, 8)
                (9, 16) (10, 16) (11, 16) (12, 32) (13, 32)
            };
            \addlegendentry{\# Heads (Parallelism)}

        \end{axis}
    \end{tikzpicture}
    \caption{Scaling of computational depth (Layers) and parallel attention capacities (Heads) across Llama models. Both metrics show a step-wise increase consistent with the model's total parameter growth.}
    \label{fig:depth_heads_scaling}
\end{figure}

\autoref{fig:depth_heads_scaling} summarizes how $L$ (layers) and $H$ (attention heads) scale with model size in our
Llama sweep. Both increase in discrete steps as the parameter count grows.

\subsection{Reasoning Prompt}
\label{apps:rp}
The following demonstrates the implicit reasoning prompts designed to address knowledge conflicts, mathematical reasoning, and multiple-choice questions, and related tasks:
\begin{tcolorbox}[colback=gray!5!white,colframe=black!75!white,title=Example Reasoning Prompt 1]
\textbf{System:} You are a helpful assistant. For each multiple choice question, you must answer with ONLY a single letter: A, B, C, D... Do not explain or add any other text.
\medskip

\textbf{User:} \{prompt\_text\}. Answer the question with only the letter:
\end{tcolorbox}

\begin{tcolorbox}[colback=gray!5!white,colframe=black!75!white,title=Example Reasoning Prompt 2]
\textbf{System:} You are a helpful Knowledge Conflict assistant. You should judge the knowledge in the sentence is correct or not. You should answer with ONLY a single letter: True or False. Do not explain or add any other text.
\medskip

\textbf{User:} \{prompt\_text\}. Answer the question with only the letter:
\end{tcolorbox}

\begin{tcolorbox}[colback=gray!5!white,colframe=black!75!white,title=Example Reasoning Prompt 3]
\textbf{System:} You are a helpful math assistant. Provide the final answer in the end.
\medskip

\textbf{User:} \{prompt\_text\}. Please provide the final answer in the end.
\end{tcolorbox}

\subsection{More Details about Knowledge Conflict and different types of In-context Learning}
\label{apps:kc}

\subsubsection{Details about Knowledge Conflict Dataset}
This dataset ~\citep{wang2024resolving} operates on a \emph{Knowledge Injection and Interference} logic through a two step. First, it establishes Parametric Knowledge by defining ground-truth facts the model acquired during pre-training, such as the standard definition of a programming variable. Second, it performs Conflict Generation using a ``shuffling'' method, where the core entity is replaced with a phonetically or conceptually similar but semantically distinct term (e.g., substituting \emph{variable} with \emph{random variable}). This creates a strategic contradiction between the provided context and the model's internal common sense.

This dataset is particularly suitable for our research because it allows for the precise control of Out-of-Distribution (OOD) intensity by manipulating the degree of external knowledge conflict. As we posit that adversarial distractors can transform an original ``easy'' input into a ``conflict'' perturbed variant, this framework provides a quantifiable way to increase task difficulty. By introducing these specific knowledge conflicts as distractor items, we can systematically simulate OOD scenarios where the model's internal parametric memory directly clashes with contradictory external evidence.

\begin{figure}
\small
\begin{tcolorbox}
[colback=gray!5!white,colframe=black!75!white,title=Knowledge Conflict Dataset Entry 1]
\textbf{Main Entity:} Variable \\
\textbf{Conflict Method:} Shuffling (Variable $\rightarrow$ Random variable)

\medskip
\hrule
\medskip

\textbf{Parametric Knowledge (Ground Truth):} \\
{\small In programming, a variable is a named container that stores a value. It can hold different data types, such as numbers or strings, and its value can be changed during the execution of a program.}

\medskip

\textbf{Conflicting Knowledge (Injected Noise):} \\
{\small In programming, a \textbf{Random variable} is a named container that stores a value. It can hold different data types, such as numbers or strings, and its value can be changed during the execution of a program.}

\medskip
\hrule
\medskip

\textbf{Generated Questions:}
What is a named container that stores a value and can be changed during the execution of a program?
\end{tcolorbox}
\end{figure}

\begin{figure}
\small
\begin{tcolorbox}[colback=gray!5!white,colframe=black!75!white,title=Knowledge Conflict Dataset Entry 2]
\textbf{Main Entity:} Equation \\
\textbf{Conflict Method:} Substitution (two $\rightarrow$ 9)

\medskip
\hrule
\medskip

\textbf{Parametric Knowledge (Ground Truth):} \\
{\small An equation is a mathematical statement that shows the equality of \textbf{two} expressions. It typically consists of variables, constants, and mathematical operations.}

\medskip

\textbf{Conflicting Knowledge (Injected Noise):} \\
{\small An equation is a mathematical statement that shows the equality of \textbf{\textcolor{GoldDetail}{9}} expressions. It typically consists of variables, constants, and mathematical operations.}

\medskip
\hrule
\medskip

\textbf{Generated Questions:}
How many expressions does an equation typically show the equality of?
\end{tcolorbox}
\end{figure}

\subsubsection{More Results on Larger Model}
\begin{figure*}[t]
    \centering 

    \begin{subfigure}[b]{0.48\textwidth}
        \centering
        \includegraphics[width=\textwidth]{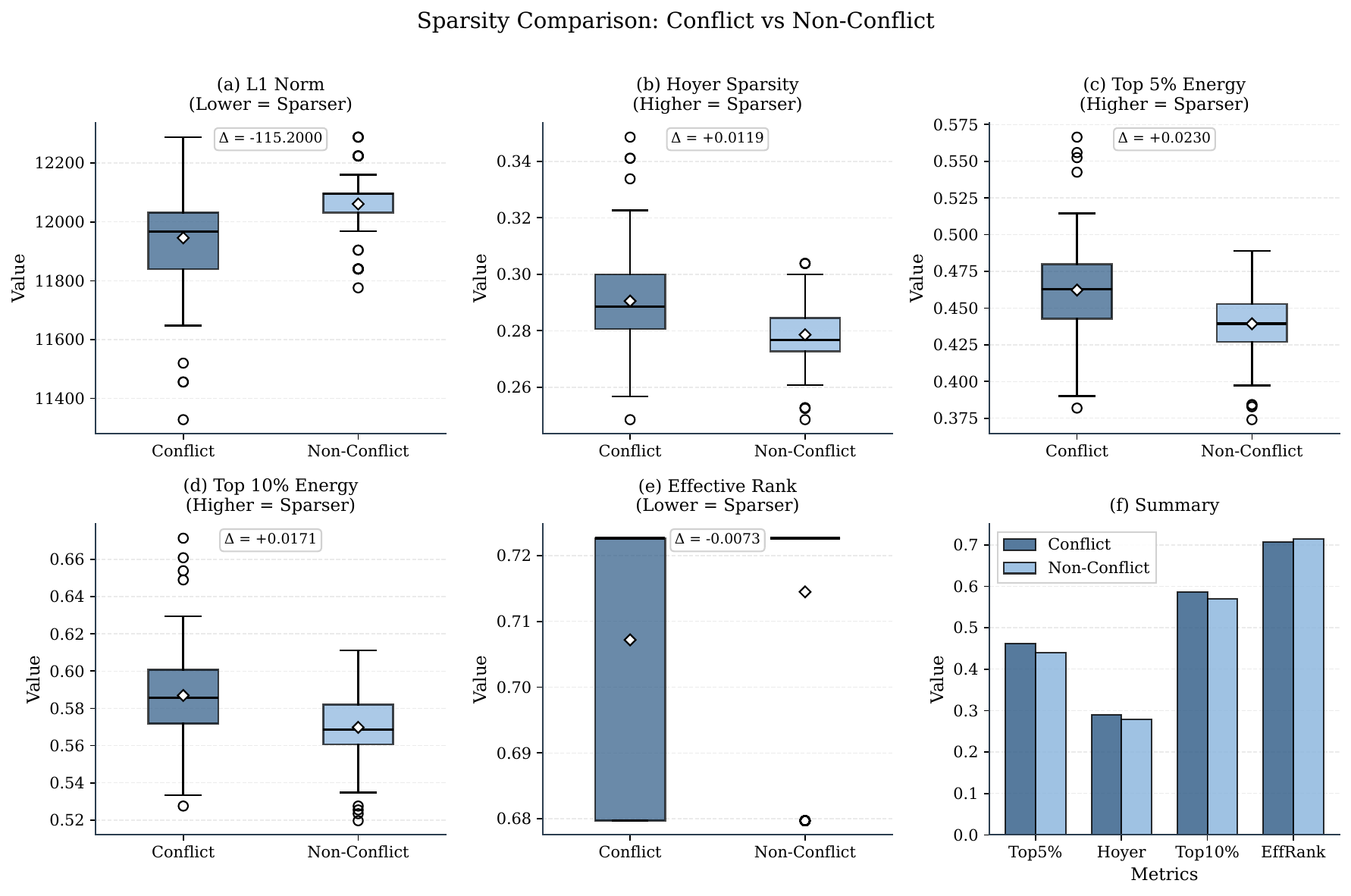}
        \caption{Llama3.1-70B}
        \label{fig:llama70b}
    \end{subfigure}

    \vspace{0.3cm}

    \begin{subfigure}[b]{0.48\textwidth}
        \centering
        \includegraphics[width=\textwidth]{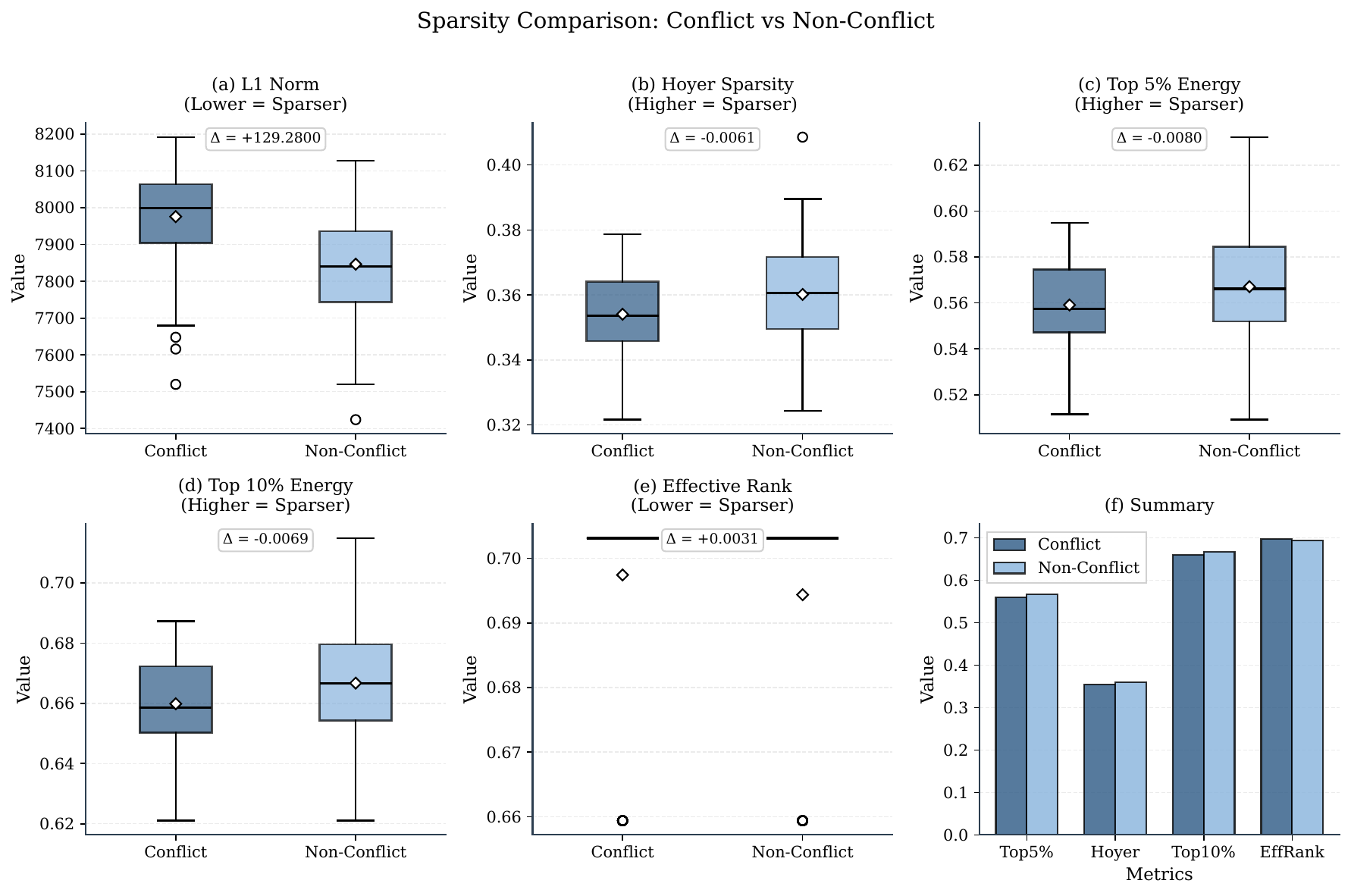}
        \caption{Qwen2.5-32B}
        \label{fig:qwen32b}
    \end{subfigure}
    \hfill 
    \begin{subfigure}[b]{0.48\textwidth}
        \centering
        \includegraphics[width=\textwidth]{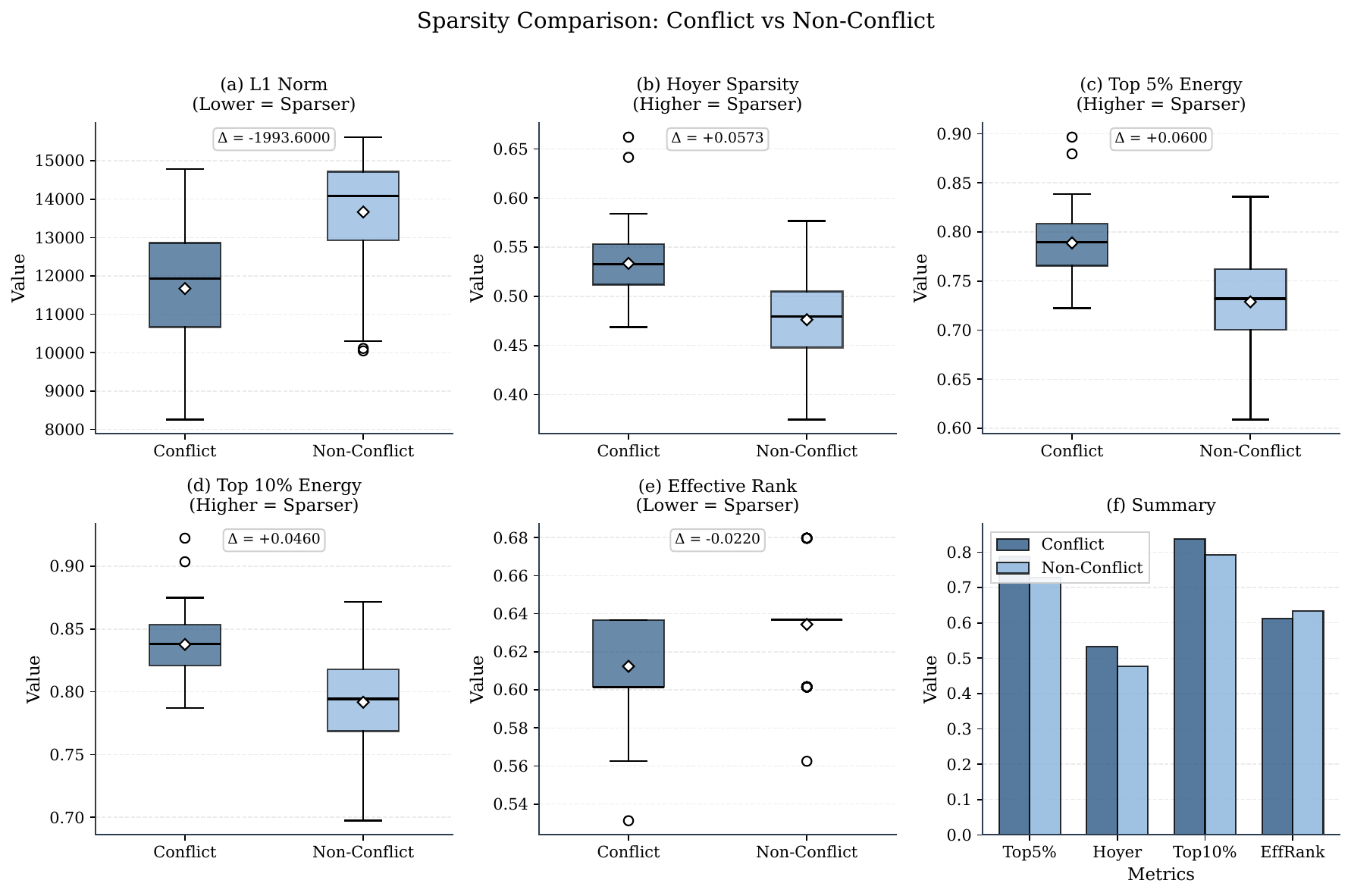}
        \caption{Qwen2.5-70B}
        \label{fig:qwen70b}
    \end{subfigure}

    \caption{\textbf{Sparsity analysis across different models.} (a) Visualizes the sparsity pattern of Llama3.1-70B. (b) and (c) compare the Qwen2.5 series models (32B and 70B).}
    \label{fig:three_models_comparison_v2}
\end{figure*}

To verify the universality of our hypothesis that OOD samples trigger higher activation sparsity in the last hidden state, we extended our evaluation to a diverse set of large-scale foundation models and used the same knowledge conflict dataset as ~\citep{wang2024resolving}. Specifically, we selected models varying in both architecture and parameter size, including \exampleb{Qwen2.5-32B}, \exampleb{Qwen2.5-70B}, and \exampleb{Llama-3.1-70B}.

As illustrated in ~\autoref{fig:three_models_comparison_v2}, the experimental results are highly consistent across all tested models. We observe that samples involving knowledge conflicts (a representative form of OOD data) systematically induce significantly sparser activation patterns compared to consistent (in-distribution) samples. This phenomenon remains robust regardless of the model's scaling or architectural design. For instance, even in the 70B parameter language models, the activation density drops sharply when the model encounters conflicting information. These findings strongly corroborate our core conclusion: activation sparsity serves as an architecture-agnostic proxy for sample difficulty and OOD status, validating the effectiveness of using sparsity for data ranking.

\subsubsection{Different Kinds of In-Context Learning}
Following the experimental framework established by ~\citep{Niu2025LlamaSL}, we devised multiple In-Context Learning (ICL) strategies to comprehensively evaluate the model's reasoning capabilities under different prompting conditions. Specifically, we implemented four distinct prompting variants, taken from ~\citep{Niu2025LlamaSL}, as shown in \autoref{tab:prompt_settings}.

\begin{table}
  \caption{Prompt Setup \citep{Niu2025LlamaSL}. We follow the exact same prompt setup as \citet{Niu2025LlamaSL}. The emojis represent the target tokens: \lie: counterfactual; \dstr: distracting; \gold: correct.}
    \vspace{0.5em}
  \centering\small
  \setlength\tabcolsep{3pt}
  \begin{tabular}{l|p{4.2cm}|p{3.9cm}|cccccc} \toprule
    Context Setting & Context Prompt & Query Prompt & \multicolumn{3}{c|}{\dstr} & \multicolumn{3}{c}{\gold} \\
    \midrule
    Related & \examplegs{On the inside, bananas are white.} & \examplegs{What color are mangoes on the inside? They are} & \multicolumn{3}{c|}{\examplegs{white}} & \multicolumn{3}{c}{\examplegs{\examplegs{orange}}} \\
    Irrelevant & \examplegs{The capital of Canada is Ottawa.} & \examplegs{What color are mangoes on the inside? They are} & \multicolumn{3}{c|}{\examplegs{Ottawa}} & \multicolumn{3}{c}{\examplegs{orange}} \\
    Random & \examplegs{Promotion} & \examplegs{What color are mangoes on the inside? They are} & \multicolumn{3}{c|}{\examplegs{Promotion}} & \multicolumn{3}{c}{\examplegs{orange}} \\
    \midrule
    \multicolumn{3}{c}{\it Distraction over Counterfactual Context} & \multicolumn{2}{|c|}{\lie} & \multicolumn{2}{c|}{\dstr} & \multicolumn{2}{c}{\gold} \\
    \midrule
    Counterfactual & \examplegs{On the inside, bananas are green.} & \examplegs{What color are mangoes on the inside? They are} & \multicolumn{2}{c|}{\examplegs{green}} & \multicolumn{2}{c|}{\examplegs{white}} & \multicolumn{2}{c}{\examplegs{orange}} \\
    \bottomrule
  \end{tabular}
  \label{tab:prompt_settings}
  \vspace{-1em}
\end{table}

\begin{table}[h]
    \caption{Sparsity metrics averaged over samples across different context settings in Qwen2.5-1.5B. }
    \vspace{0.5em}
    \centering
    \small
    \setlength{\tabcolsep}{6pt} 
    \begin{tabular}{lcccc}
        \toprule
        \textbf{Context Setting} & \textbf{$L_1$ Norm} & \textbf{Top-5\% Energy} & \textbf{Top-10\% Energy} & \textbf{Eff. Rank} \\
        \midrule
        Related & 5066.96 & 0.5723 & 0.6715 & 0.6715 \\
        Irrelevant & 5072.15 & 0.5816 & 0.6772 & 0.6764 \\
        Random & 3689.13 & 0.7600 & 0.8203 & 0.5987 \\
        Counterfactual & 4938.74 & 0.5978 & 0.6937 & 0.6680 \\
        \bottomrule
    \end{tabular}
    \label{tab:sparsity_stats}
\end{table}

We categorize these prompting strategies into two primary dimensions: \textit{Contextual Distraction} and \textit{Counterfactual Conflict}. The Contextual Distraction settings comprise Related, Irrelevant, and Random contexts, where the input provides truthful but potentially noisy information, aiming to measure the model's susceptibility to blindly copying context tokens versus accessing its internal parametric knowledge. Specifically, the `Related` context tests the capability to distinguish between semantically similar attributes within a domain, whereas `Irrelevant` and `Random` contexts serve as baselines to assess the model's resilience against pure recency bias and token repetition. Conversely, the Counterfactual Conflict setting introduces a direct contradiction between the context and ground truth, creating a tripartite competition in the probability space. Following this categorization, we further investigated the sparsity patterns of the model's last hidden state representations under these distinct settings, as shown in~\autoref{tab:sparsity_stats}.

Our results provide empirical support for the hypothesis that inputs deviating further from the pre-training distribution (OOD) induce higher activation sparsity. We observe a clear gradient of sparsity corresponding to the degree of distribution shift. These metrics collectively validate that sparsity is positively correlated with the OOD nature of the input. As the context shifts from natural language (Related) to semantic conflict (Counterfactual) and finally to complete noise (Random), the model's last hidden state representations become increasingly sparse and low-rank.

As shown in \autoref{tab:sparsity_stats}, our results provide empirical support for the hypothesis that inputs deviating further from the pre-training distribution (OOD) induce higher activation sparsity. The \texttt{Random} setting, representing the most severe OOD scenario characterized by semantic and structural noise, exhibits distinct hyper-sparsity characteristics. It records the lowest activation magnitude ($L_1 \approx 3689.13$) while simultaneously showing the highest energy concentration (Top-5\% Energy $\approx 0.7600$), indicating that nonsensical inputs fail to activate broad semantic circuits, thereby concentrating signal flow into a narrow subset of neurons and collapsing the representation into a lower-dimensional subspace as evidenced by the significantly lower Effective Rank ($0.5987$). A more subtle trend is observed when comparing \texttt{Counterfactual} to \texttt{Related} contexts; although both are grammatically well-formed, the semantic conflict in the \texttt{Counterfactual} setting induces slightly higher sparsity than \texttt{Related} inputs ($L_1$: $4938.74 < 5066.96$; Top-5\% Energy: $0.5978 > 0.5723$). This suggests that processing ``lies'' or counter-knowledge conflicts engages fewer knowledge retrieval circuits compared to processing truthful, coherent information, effectively ``pruning'' the activation of conflicting parametric knowledge. Conversely, the \texttt{Related} and \texttt{Irrelevant} settings serve as in-distribution baselines, showing the highest $L_1$ norms and lowest energy concentrations, which reflect a dense, broad activation of the model's semantic networks in response to natural linguistic structures.

\subsection{More Details about Reasoning Complexity}
\label{apps:complex}
To further understand the relationship between sparsity and model performance, we examine how prediction accuracy varies with sparsity across different difficulty levels. As shown in \autoref{fig4}, we observe a strong statistical correlation between accuracy and the sparsity of the last hidden states. Specifically, on the MATH-500~\citep{lightman2024lets} dataset with Qwen2.5-7B, accuracy is highly correlated with both $\ell_1$ norm sparsity ($r=0.897$) and Top-10\% Energy ($r=-0.915$).

\begin{figure}[h]
    \centering
    \includegraphics[width=0.78\linewidth]{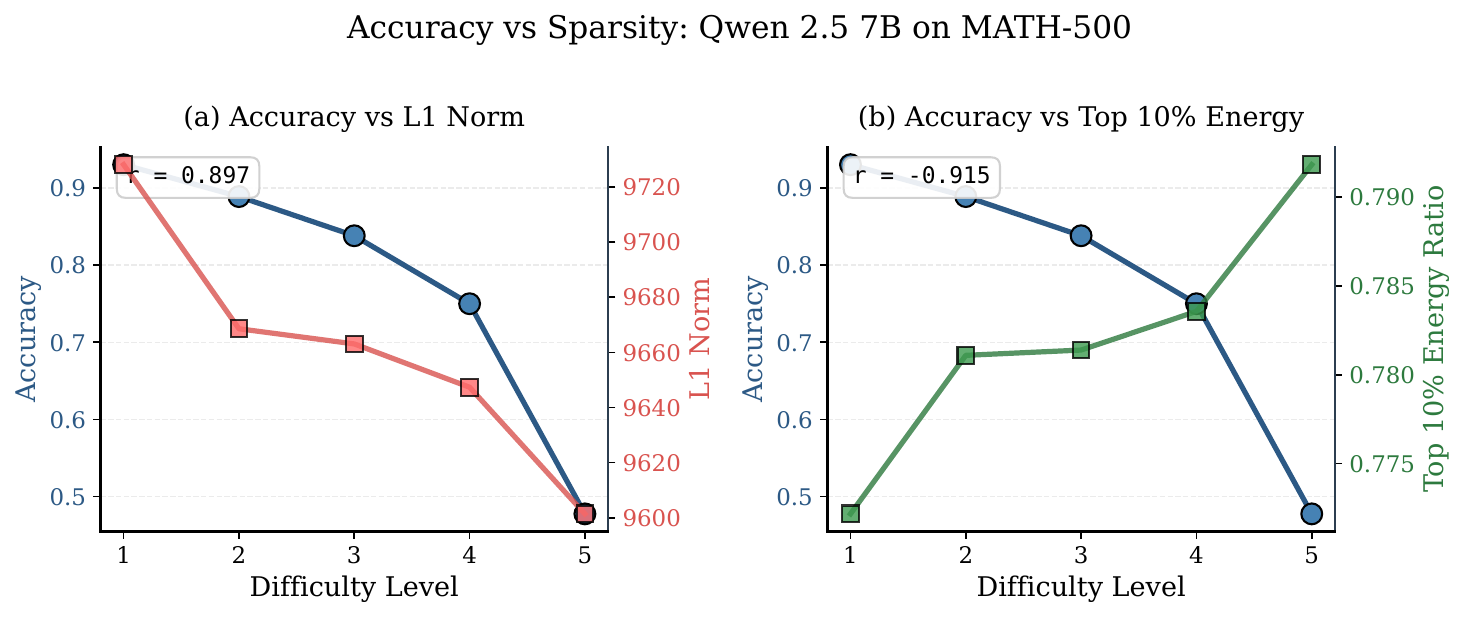}
    \caption{\textbf{Accuracy–Sparsity Correlation.} Accuracy strongly correlates with Last Hidden State sparsity on MATH-500 for Qwen2.5-7B ($L_1$: $r$=0.897; Top-10\% Energy: $r$=-0.915).}
    \label{fig4}
\end{figure}

Together, Figure~\ref{fig3} and Figure~\ref{fig4} present a consistent statistical pattern. Figure~\ref{fig3} shows that higher reasoning difficulty systematically induces stronger sparsity in the last hidden states, while Figure~\ref{fig4} demonstrates that higher sparsity is tightly associated with lower model accuracy. In other words, tasks that induce stronger representational sparsity tend to correspond to lower prediction accuracy. These results suggest a coherent relationship between reasoning difficulty, representational sparsity, and model performance: as tasks become harder, the model's last hidden states become progressively sparser, and this sparsity is accompanied by a corresponding decline in accuracy.

We use DeepMath-103K~\citep{he2026deepmathk}, a large-scale math reasoning dataset with about 103K examples. Each sample contains a question, final answer, topic annotation, multiple reference solutions, and an explicit difficulty score, which allows us to partition the data into different difficulty levels (from 1-10) for controlled analysis.

\begin{tcolorbox}[breakable,colback=gray!5,colframe=black!40,title=Example from DeepMath-103K]
\textbf{Topic:} Mathematics $\to$ Precalculus $\to$ Limits \\
\textbf{Difficulty:} 4.5

\textbf{Question.}
Evaluate the limit
\[
\lim_{x \to \infty} \sqrt{x}\left(\sqrt[3]{x+1}-\sqrt[3]{x-1}\right).
\]

\textbf{Solution.}
Let
\[
a=(x+1)^{1/3}, \qquad b=(x-1)^{1/3}.
\]
Then
\[
\begin{aligned}
&a-b\\
&= \frac{a^3-b^3}{a^2+ab+b^2} \\
&= \frac{(x+1)-(x-1)}{a^2+ab+b^2} \\
&= \frac{2}{a^2+ab+b^2}.
\end{aligned}
\]
Hence
\[
\sqrt{x}(a-b)=\frac{2\sqrt{x}}{a^2+ab+b^2}.
\]
As $x\to\infty$, we have
\[
a^2\sim x^{2/3}, \qquad ab\sim x^{2/3}, \qquad b^2\sim x^{2/3},
\]
so
\[
a^2+ab+b^2 \sim 3x^{2/3}.
\]
Therefore,
\[
\sqrt{x}(a-b)\sim \frac{2x^{1/2}}{3x^{2/3}}
=\frac{2}{3}x^{-1/6}\to 0.
\]

\textbf{Final Answer.} 0
\end{tcolorbox}

DeepMath-103K is particularly suitable for our study because its explicit difficulty stratification enables controlled, dataset-level analysis of whether increasing reasoning difficulty systematically leads to sparser last-layer representations. This pattern is further confirmed by our experiments on Qwen2.5-3B and Qwen2.5-1.5B, where we consistently observe that increasing difficulty is accompanied by progressively sparser last-layer representations across difficulty strata.

Across metrics, harder problems tend to induce progressively sparser final-layer representations, showing a consistent dataset-level relationship between reasoning difficulty and representational sparsity (as Figure~\ref{fig:deepmath_sparsity} shows).

\begin{figure*}[t]
    \centering
    \begin{minipage}{0.49\textwidth}
        \centering
        \includegraphics[width=\linewidth]{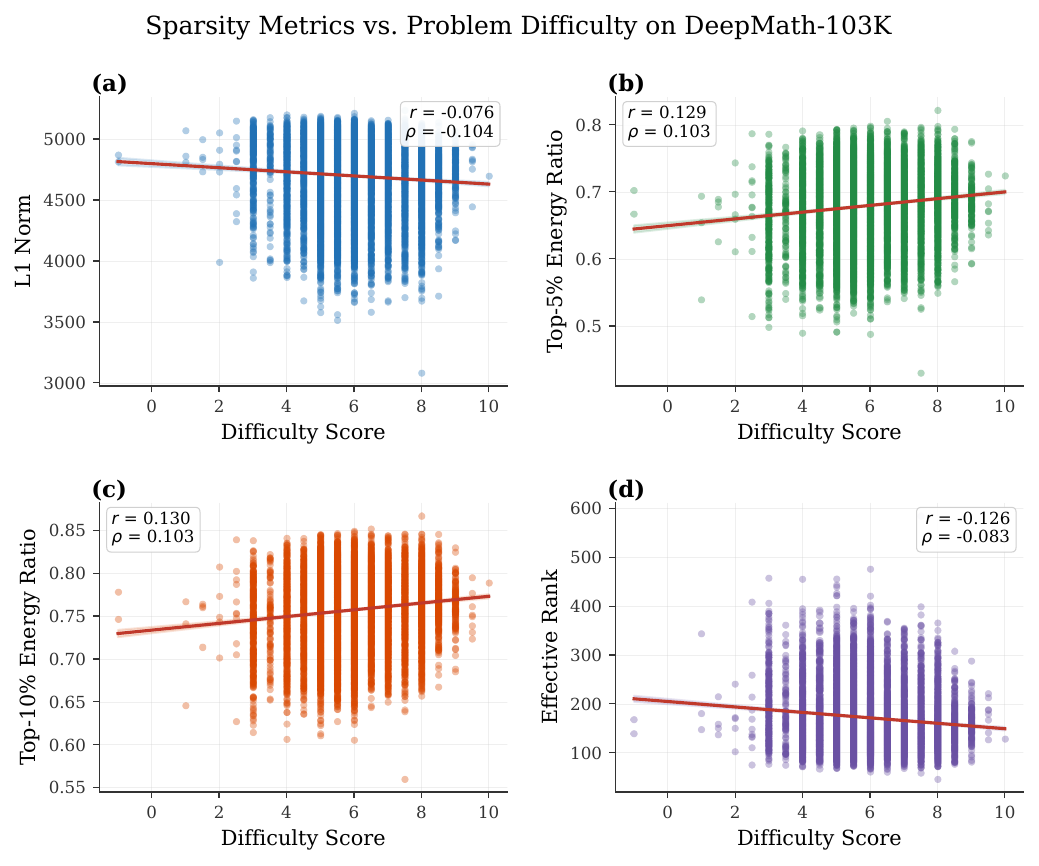}
        \centerline{\small (a) Qwen2.5-3B-Instruct}
    \end{minipage}
    \hfill
    \begin{minipage}{0.49\textwidth}
        \centering
        \includegraphics[width=\linewidth]{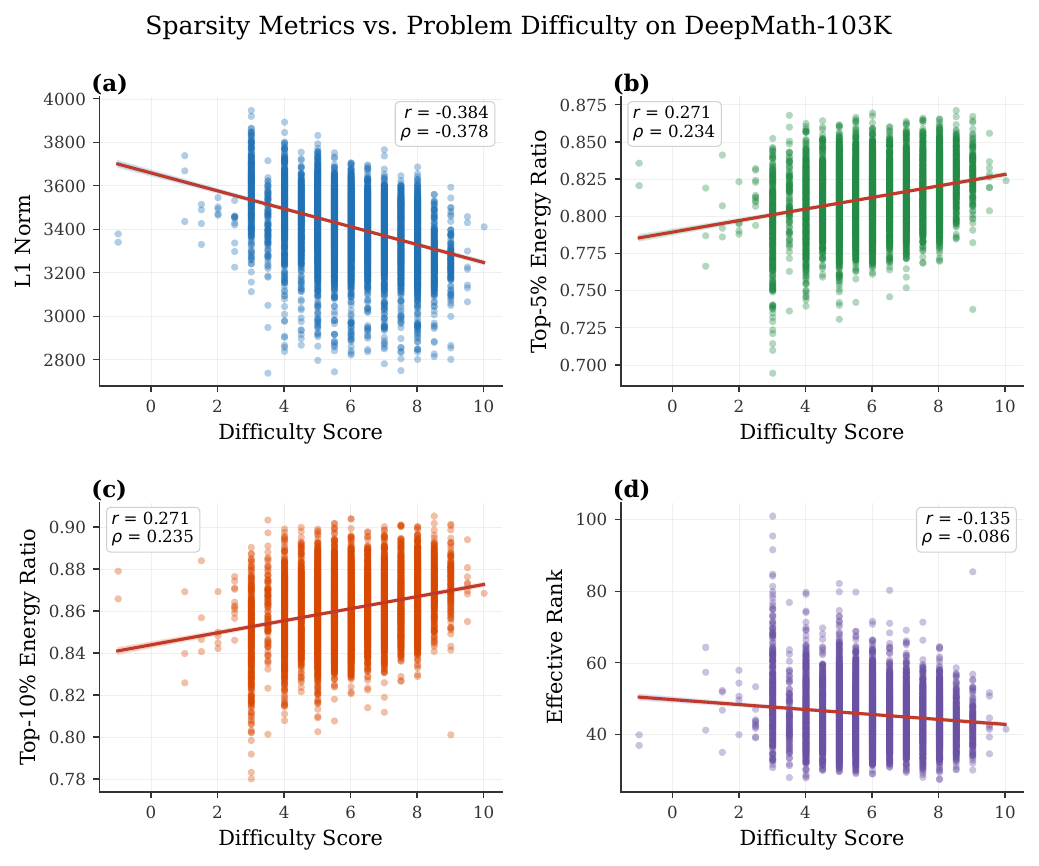}
        \centerline{\small (b) Qwen2.5-1.5B-Instruct}
    \end{minipage}

    \caption{
    \textbf{Sparsity Metrics vs. Problem Difficulty on DeepMath-103K.}
    We plot several last-layer sparsity metrics against problem difficulty levels for Qwen2.5-3B and Qwen2.5-1.5B. Across both models, harder problems tend to induce progressively sparser final-layer representations, showing a consistent dataset-level relationship between reasoning difficulty and representational sparsity.
    }
    \label{fig:deepmath_sparsity}
    \vspace{-1em}
\end{figure*}

\subsection{Layer Selection for Sparsity Measurement}
\label{apps:layerproblem}

\begin{figure}[h]
    \centering
    \includegraphics[width=0.5\linewidth]{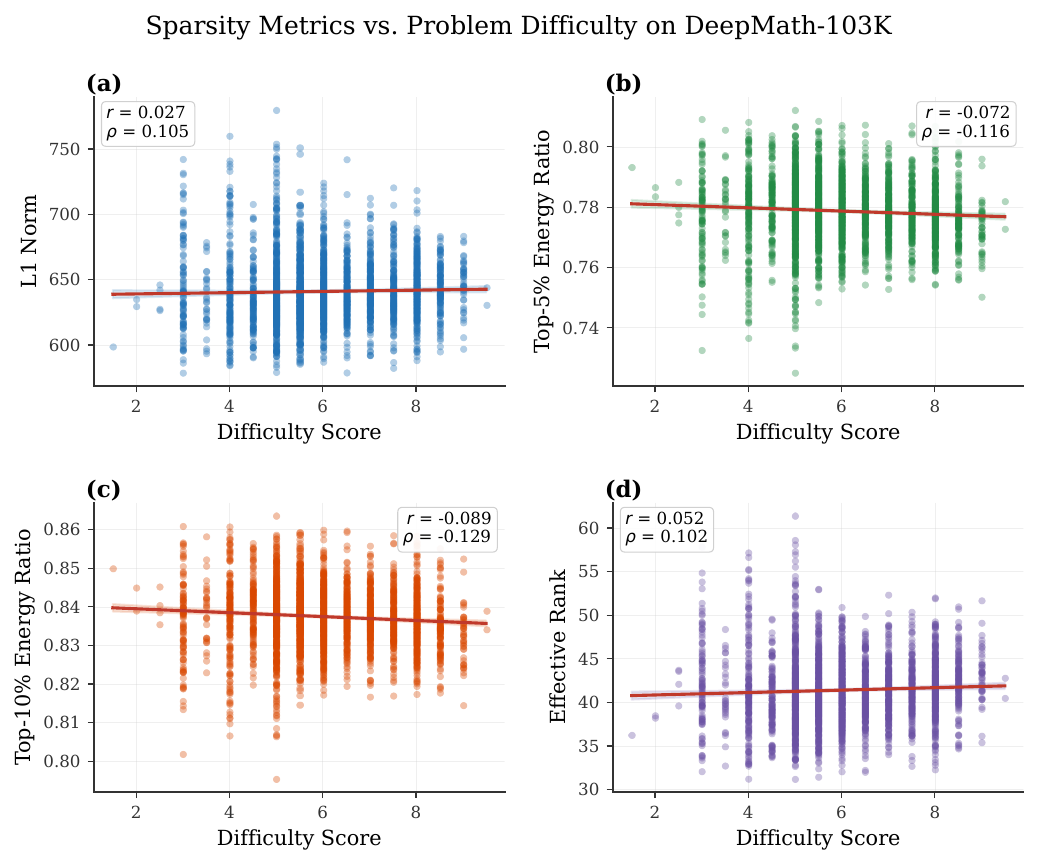}
\caption{\textbf{Sparsity Metrics vs. Problem Difficulty on DeepMath-103K for Qwen2.5-1.5B in Layer8.} We plot several 8-layer sparsity metrics against problem difficulty levels on DeepMath-103K. }
    \label{fig:deepmath_example3}
\end{figure}

An important design choice is \emph{which layer} should be used to measure sparsity. Our main analysis focuses on the \textbf{last hidden state}, and this choice is empirically motivated rather than arbitrary. Intuitively, if sparsity is tied to the model's final adaptation to problem difficulty, the effect should become most visible near the output of the network, where reasoning-relevant information is most consolidated. To verify this, we perform an additional layer-wise comparison on \textbf{DeepMath-103K}. We find that the sparsity--difficulty relationship is much clearer at the \textbf{final layer}, whereas earlier or intermediate layers often show only weak trends or even \emph{inconsistent directions}. For example, at the final layer, increasing difficulty tends to decrease $\ell_1$ norm and effective rank while increasing Top-$k$ energy, which is consistent with our sparsification hypothesis. In contrast, at \textbf{Layer 8}, these trends become much weaker, and several metrics even reverse direction. This suggests that intermediate-layer sparsity is not a reliable proxy for difficulty, and that the decisive sparsification effect only emerges near the end of the network. The result is also shown in Table~\ref{tab:layer_choice_deepmath}.

Therefore, throughout the main paper, we use the \textbf{last hidden state} as the primary object of analysis. Earlier layers may still encode useful information, but they do not provide a stable or interpretable sparsity signal for difficulty tracking.

\begin{table*}[h]
\centering
\caption{
DeepMath-103K correlations between problem difficulty and sparsity metrics.
Each entry reports Pearson $r$ / Spearman $\rho$.
The expected sparsification pattern is lower effective rank and higher Top-$k$ energy as difficulty increases.
}
\label{tab:layer_choice_deepmath}
\vspace{-0.4em}
\scriptsize
\setlength{\tabcolsep}{3pt}
\renewcommand{\arraystretch}{1.05}
\resizebox{0.9\textwidth}{!}{
\begin{tabular}{@{}lcccc@{}}
\toprule
\textbf{Setting}
& $\boldsymbol{\ell_1}$
& \textbf{Top-5\%} $\uparrow$
& \textbf{Top-10\%} $\uparrow$
& \textbf{Eff. Rank} $\downarrow$ \\
\midrule

Last layer, 3B
& $-0.076 / -0.104$
& $0.129 / 0.103$
& $0.130 / 0.103$
& $-0.126 / -0.083$ \\

Last layer, 1.5B
& $-0.384 / -0.378$
& $0.271 / 0.234$
& $0.271 / 0.235$
& $-0.135 / -0.086$ \\

Layer 8, 1.5B
& $0.027 / 0.105$
& $-0.072 / -0.116$
& $-0.089 / -0.129$
& $0.052 / 0.102$ \\

\bottomrule
\end{tabular}
}
\vspace{-0.8em}
\end{table*}

\subsection{More Details about MMLU-Robust}
\label{apps:mmlu}
\subsubsection{Example in MMLU-PRO and MMLU-Robust}
MMLU-Pro~\citep{wang2024mmlu} is an enhanced version of the original Massive Multitask Language Understanding (MMLU) benchmark~\citep{hendrycks2021measuring}, developed to address shortcomings that have emerged as modern LLMs have improved. In particular, the original MMLU has become increasingly \emph{too easy} for state-of-the-art models and is affected by data noise in certain subsets. MMLU-Pro introduces several key improvements over its predecessor.
\ding{117} \textbf{Increased difficulty.} Overly simple factual questions are filtered out and replaced with more challenging problems that require multi-step and rigorous reasoning.
\ding{117} \textbf{Expanded answer space.} The number of multiple-choice options is increased from four to ten, substantially reducing the likelihood of correct answers due to random guessing and better reflecting genuine reasoning ability.
\ding{117} \textbf{Broader coverage.} The benchmark comprises over 12,000 questions across 14 core academic domains, spanning STEM, humanities, and social sciences.

To systematically evaluate the last hidden state under progressively intensifying OOD conditions, we construct a variant dataset, \textbf{MMLU-Robust}. Building on the already challenging MMLU-Pro benchmark, which features ten multiple-choice options, we introduce an iterative distractor augmentation mechanism. The idea is simple: adding more plausible options makes the task harder. This design draws on prior work that probes LLM robustness through adversarial input perturbations, including MathAttack and prompt-injection attacks~\citep{zhou2024mathattack, zhang2024goal, zhang2024target}.

In particular, we employ a perturbation function $\mathcal{P}(\cdot)$ to generate plausible but incorrect distractors derived from the existing option set. For a given question with an initial option set $\mathcal{O}_{10}$ ($|\mathcal{O}_{10}| = 10$), we select a subset of non-ground-truth options and apply semantic or numerical perturbations (e.g., sign flipping, unit alteration, or logical inversion) to create new adversarial options. These perturbed variants are injected back into the pool to form expanded sets $\mathcal{O}_{15}$ and $\mathcal{O}_{20}$.

This process creates a controlled ``interference gradient.'' By increasing the option count from 10 to 15 and finally to 20, we densify the solution space with highly correlated noise, forcing the model to discern increasingly subtle differences. We hypothesize that this artificially induced complexity acts as a proxy for severe OOD shifts, triggering the sparse representation mechanism described in our main findings.
\begin{figure}
\small
\begin{tcolorbox}[
    colback=SoftGray,
    colframe=DeepSlate,
    coltitle=white,
    fonttitle=\bfseries\sffamily,
    sharp corners,
    boxrule=0.8pt,
    title=MMLU-Pro Dataset Samples
]

\textbf{\textcolor{DeepSlate}{Sample 1: Number Theory}} \\
\textbf{Question:} Let $A$ be the set of all ordered pairs of integers $(m, n)$ such that $7m + 12n = 22$. What is the greatest negative number in the set $B = \{m + n : (m, n) \in A\}$?

\vspace{0.5em}
\textbf{Options:}
\begin{enumerate}[label=(\Alph*), itemsep=0pt, parsep=0pt, topsep=2pt, leftmargin=2em]
    \item[] (A) -5 \quad (B) 0 \quad (C) -3 \quad
    \item[] (D) -7 \quad
    \item[] (E) \textbf{-4} \textit{ (Correct)} \quad (F) -6 \quad
    \item[]  (G) -1 \quad (H) -2 \quad (I) -9 \quad (J) N/A
\end{enumerate}
\vspace{0.5em}
\textbf{\small [CoT Reasoning]} {\small We have $12n = 22 - 7m$. A particular solution is $m = -2, n = 3$, so $m+n=1$. The general solution is $m = -2 + 12k, n = 3 - 7k$. Then $m+n = 1 + 5k$. For $m+n < 0$, we need $5k < -1$, so $k \le -1$. Max value is at $k=-1$: $1 + 5(-1) = -4$.}

\end{tcolorbox}
\end{figure}

\begin{figure}
\small
\begin{tcolorbox}[
    colback=SoftGray,
    colframe=DeepSlate,
    coltitle=white,
    fonttitle=\bfseries\sffamily,
    sharp corners,
    boxrule=0.8pt,
    title=MMLU-Robust Dataset Samples
]
\textbf{\textcolor{DeepSlate}{Sample 1: Number Theory}} \\
\textbf{Question:} Let $A$ be the set of all ordered pairs of integers $(m, n)$ such that $7m + 12n = 22$. What is the greatest negative number in the set $B = \{m + n : (m, n) \in A\}$?

\vspace{0.5em}
\textbf{Options:}
\begin{enumerate}[label=(\Alph*), itemsep=0pt, parsep=0pt, topsep=2pt, leftmargin=2em]
   \item[] (A) -5 \quad (B) 0 \quad (C) -3 \quad
      \item[] (D) -7 \quad
    \item[] (E) \textbf{-4} \textit{ (Correct)} \quad (F) -6 \quad
    \item[]  (G) -1 \quad (H) -2 \quad (I) -9 \quad (J) N/A
    \item[] (H) \textbf{{1}} \quad \textit{\scriptsize [Valid member, but positive]}
    \item[] (I) \textbf{-14} \quad \textit{\scriptsize [Valid member, but < -4]}
    \item[] (J) \textbf{4} \quad \textit{\scriptsize [Sign flip]}
    \item[] (K) \textbf{-8} \quad \textit{\scriptsize [Dense noise]}
    \item[] (M) \textbf{125} \quad \textit{\scriptsize [Dense noise]}
\end{enumerate}

\end{tcolorbox}
\end{figure}

\subsubsection{More Results in MMLU-Robust}
\label{app:mmlu_results}

\begin{figure}
    \centering
    \includegraphics[width=\linewidth]{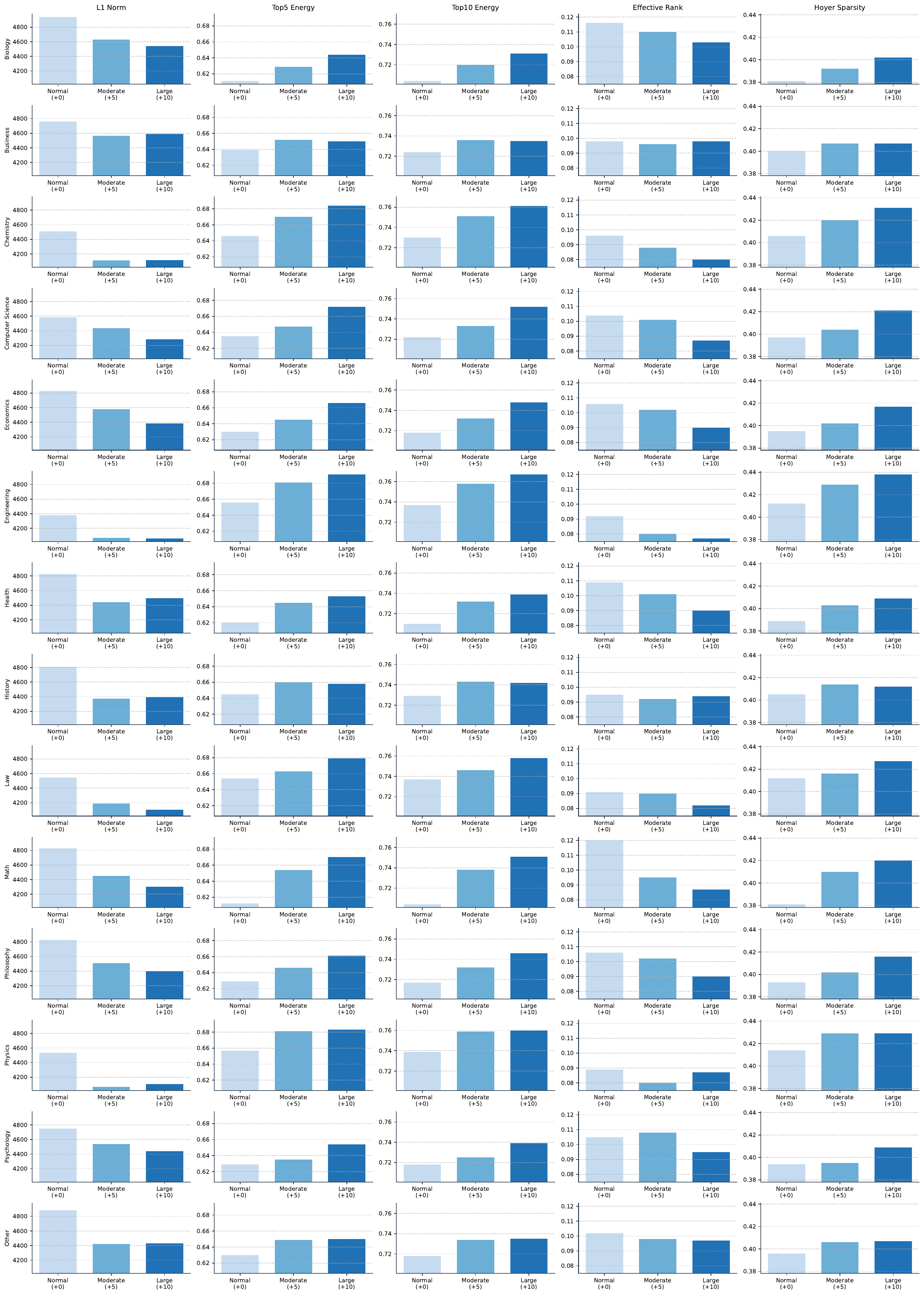}
    \caption{Per-domain sparsity statistics for \exampleb{Qwen2.5-3B-Instruct} on MMLU-Robust.}

    \label{fig:qwen_breakdown}
\end{figure}

\begin{figure}
    \centering
    \begin{subfigure}[t]{\linewidth}
        \centering
        \includegraphics[width=\linewidth]{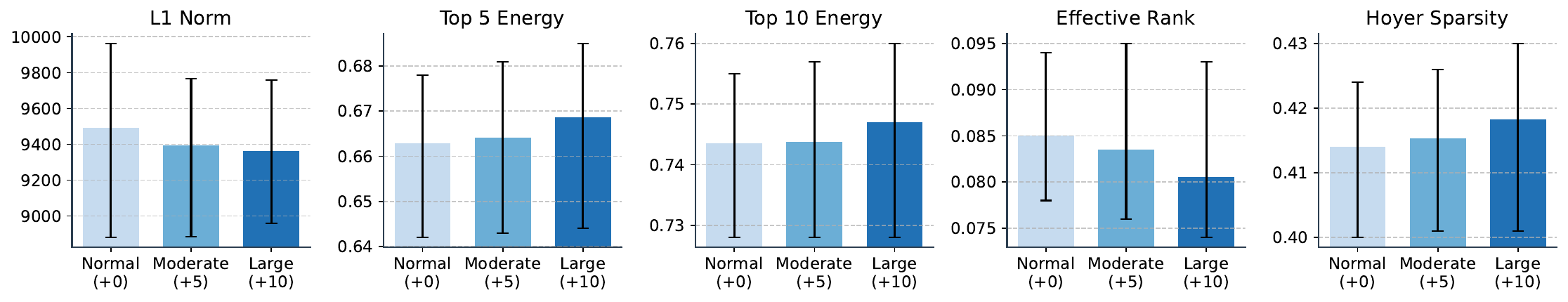}
        \caption{\exampleb{Qwen2.5-7B-Instruct}}
    \end{subfigure}

    \vspace{0.5em}

    \begin{subfigure}[t]{\linewidth}
        \centering
        \includegraphics[width=\linewidth]{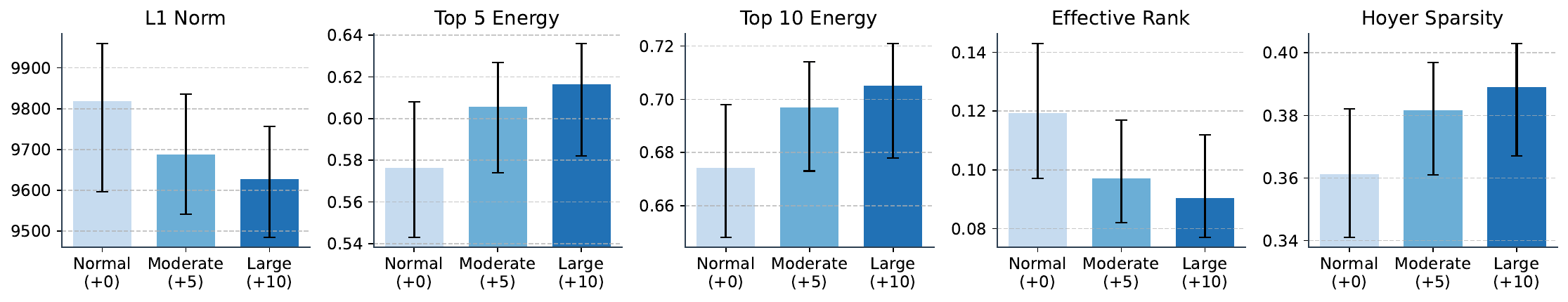}
        \caption{\exampleb{Qwen2.5-14B-Instruct}}
    \end{subfigure}

    \vspace{0.5em}

    \begin{subfigure}[t]{\linewidth}
        \centering
        \includegraphics[width=\linewidth]{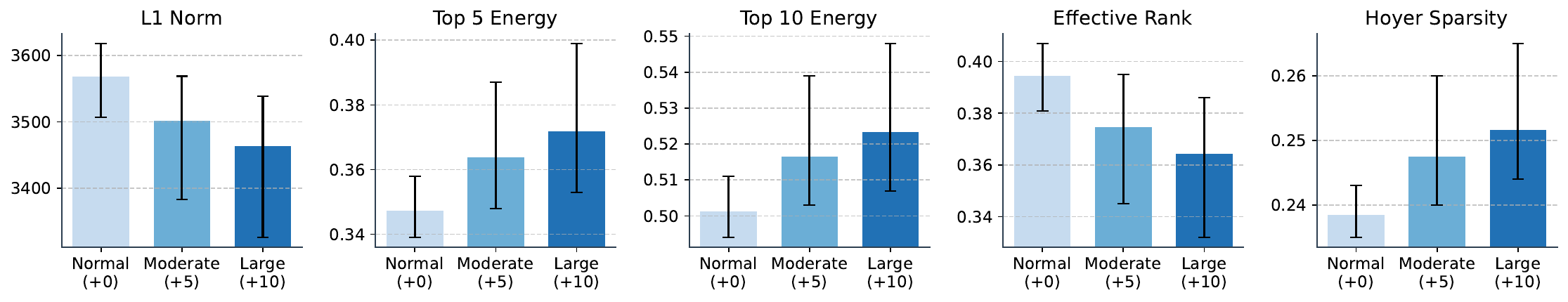}
        \caption{\exampleb{Llama-3.2-1B-Instruct}}
    \end{subfigure}

    \vspace{0.5em}

    \begin{subfigure}[t]{\linewidth}
        \centering
        \includegraphics[width=\linewidth]{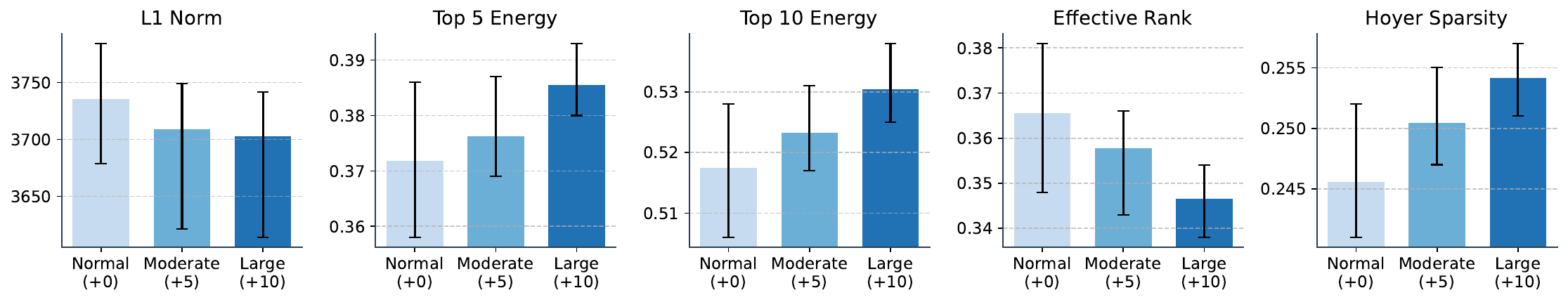}
        \caption{\exampleb{Llama-3.2-3B-Instruct}}
    \end{subfigure}

    \caption{Average sparsity statistics on MMLU-Robust for more LMs.}
    \label{fig:qwen_large_mmlu}
\end{figure}


\autoref{fig:qwen_breakdown} reports a fine-grained breakdown of sparsity-related results on MMLU-Robust for \exampleb{Qwen2.5-3B-Instruct}, grouped by academic area and evaluated under three adversarial noise levels (Normal (+0) / Moderate (+5) / Hard (+10)). For each area, we aggregate the last hidden state statistics across questions and visualize five metrics: normalized $\ell_1$ magnitude, Top-$5\%$ Energy, Top-$10\%$ Energy, Effective Rank, and Hoyer Sparsity.

\autoref{fig:qwen_large_mmlu} reports averaged results for four additional LMs. Across these models, we observe the same qualitative pattern as in the main results: increasing OOD difficulty consistently leads to sharper sparsity in the last hidden state across domains and metrics. This consistency suggests that the observed behavior is not specific to a particular model size or architecture. We encourage readers to reproduce and extend our analysis on additional models, datasets, and evaluation settings to further test the generality of this phenomenon.

To address the space limitation in the main text, we further report a per-area analysis on MMLU(-Robust/MMLU-Pro), where we evaluate each academic field separately to verify whether the sparsity--difficulty relationship holds beyond overall averages. As shown in \autoref{fig:qwen_breakdown}, the trend remains robust across domains: as adversarial noise increases (Normal $\rightarrow$ Moderate $\rightarrow$ Hard), the representation becomes consistently \emph{more sparse} in nearly every area. Concretely, harder (more perturbed) inputs exhibit lower normalized $\ell_1$ magnitude and effective rank, together with higher Top-$k$ energy concentration and Hoyer sparsity. While the absolute metric levels vary by subject, the direction of change is highly consistent, supporting the conclusion that the ``harder-is-sparser'' phenomenon is not driven by a small subset of domains but persists broadly across academic areas.

\subsection{Long Context Reasoning Details}
\label{apps:long}
\begin{figure}[t]
    \centering
    \includegraphics[width=0.6\linewidth]{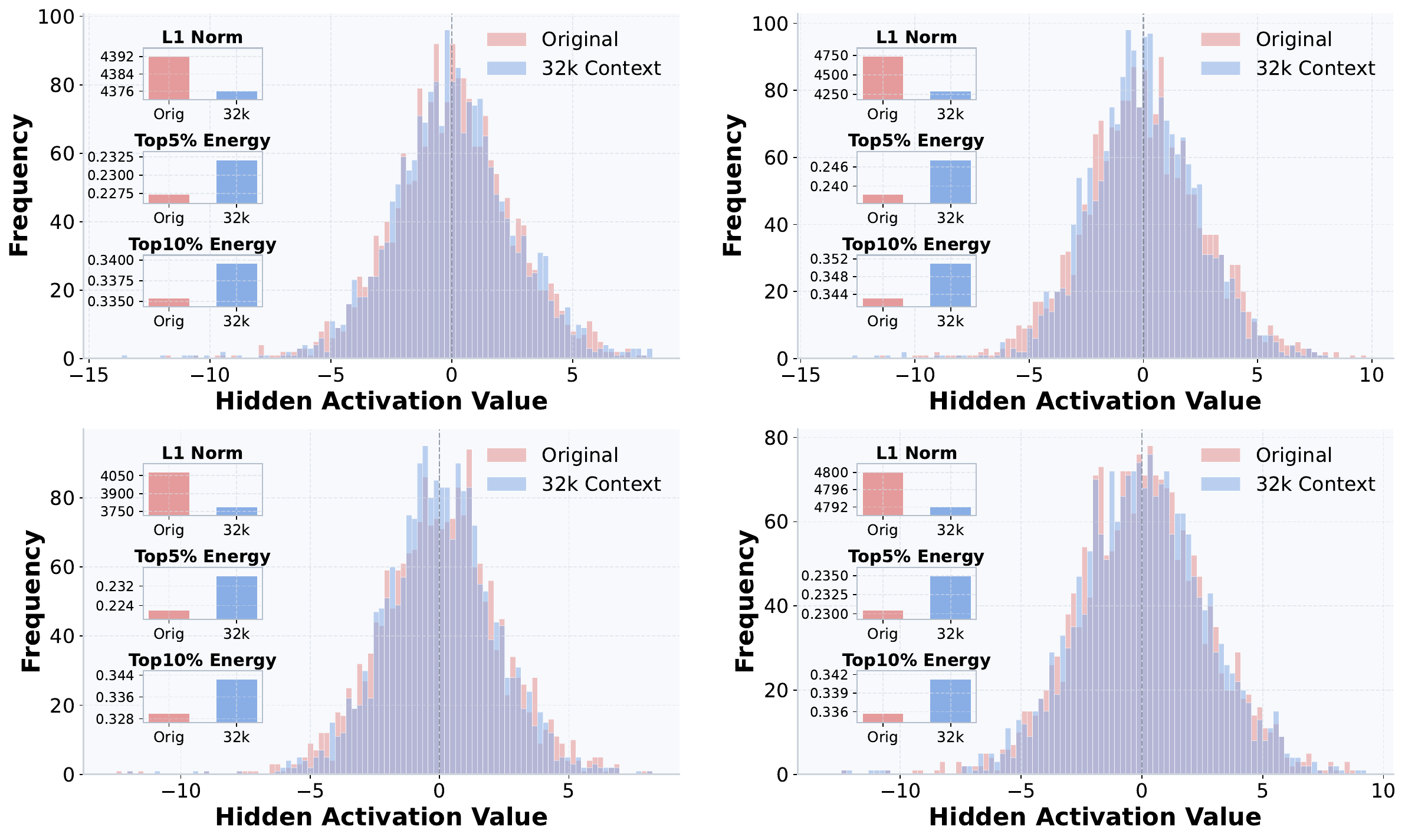}
    \caption{\textbf{Task complexity induces sparsity.} Representation activation showing increased sparsity in harder tasks \textcolor{blue}{(32k context, blue)} compared to original settings \textcolor{red}{(Short context, red)}. Three metrics consistently demonstrate higher sparsity for challenging inputs.}
    \label{fig1}
    \vspace{-1em}
\end{figure}

\begin{figure*}[!t]
    \centering
    \includegraphics[width=\textwidth]{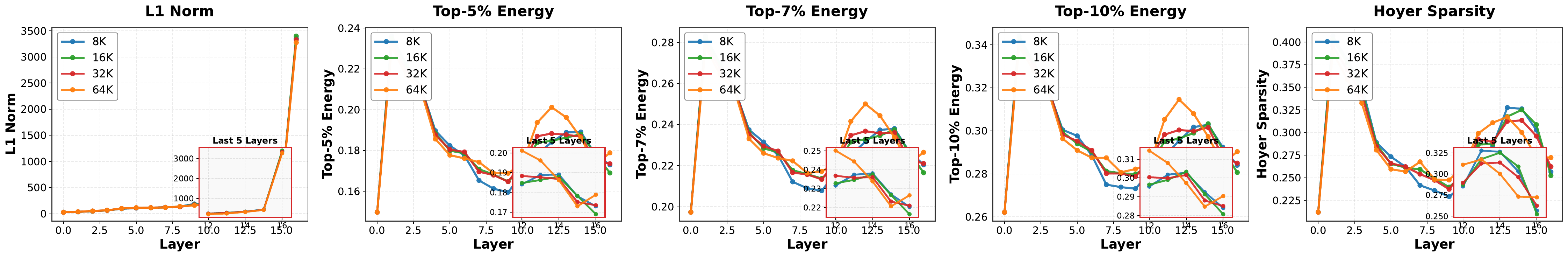}
    \parbox{0.95\linewidth}{
      \centering
      \textcolor{blue2}{\textbf{\texttt{\scriptsize
      CONTEXT: Sparsity metrics across all layers for Llama3.2-1B under varying context lengths (8K, 16K, 32K, 64K).
      }}}%
    }
    \vspace{0.3em}
    \caption{\textbf{Layer-wise Sparsity across Context Lengths.} While intermediate layers show minimal variation across contexts, the final layers exhibit sharp divergence: longer contexts consistently produce sparser representations. This experiment is conducted on LongReasonQA~\citep{li2025longcontext}, which can control the background context length.}
    \label{fig10}
\end{figure*}

\begin{figure*}[!t]
    \centering
    \includegraphics[width=\textwidth]{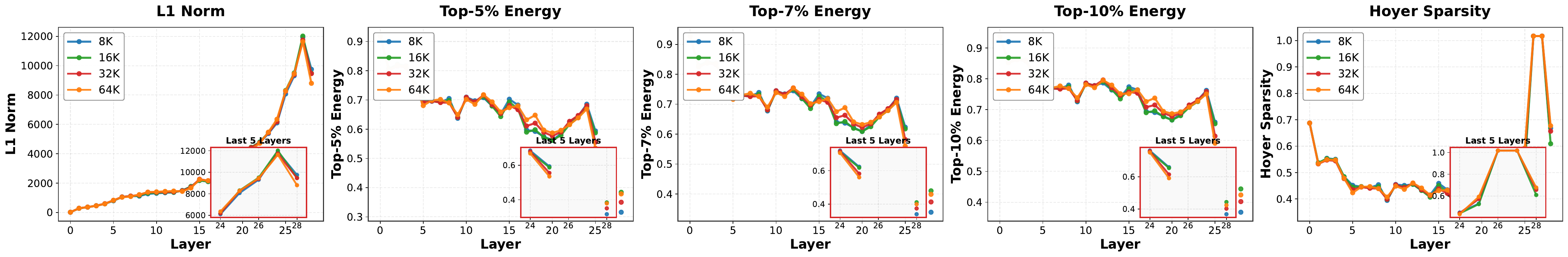}
    \parbox{0.95\linewidth}{
      \centering
      \textcolor{blue2}{\textbf{\texttt{\scriptsize
      CONTEXT: Sparsity metrics across all layers for Qwen2.5-7B under varying context lengths (8K, 16K, 32K, 64K).
      }}}%
    }
    \vspace{0.3em}
    \caption{\textbf{Layer-wise Sparsity across Context Lengths.} While intermediate layers show minimal variation across contexts, the final layers exhibit sharp divergence: longer contexts consistently produce sparser representations. This experiment is conducted on LongReasonQA~\citep{li2025longcontext}, which can control the background context length.}
    \label{fig10b}
\end{figure*}

To further probe whether our ``harder-is-sparser'' finding extends beyond adversarial perturbations, we study a
\emph{long-context reasoning} setting where difficulty is increased by expanding the background context length.
Concretely, we compare short-context prompts against a long-context (32K) variant constructed from LongReasonQA, which
controls the amount of background evidence while keeping the core question unchanged. This setup isolates the effect of
contextual complexity (more tokens, more distractors, and longer-range dependency) on internal representations.

\autoref{fig1} visualizes the distribution of last-layer activations under short vs.\ 32K context. The long-context
condition exhibits a visibly more concentrated distribution, and the accompanying sparsity metrics shift consistently in
the ``sparser'' direction: the normalized $\ell_1$ magnitude decreases, while Top-$k$ energy concentration increases.
These results mirror the robustness experiments in the main text, suggesting that increasing task complexity via longer
contexts induce a sharper, more selective representation at inference time.

We further examine \emph{where} this sparsification emerges across the depth of the LLM. \autoref{fig10} reports
layer-wise sparsity metrics for \exampleb{Llama-3.2-1B-Instruct} across 8K/16K/32K/64K contexts. A clear pattern appears: intermediate layers remain largely stable across context lengths, whereas the final layers diverge sharply, with longer contexts producing consistently sparser activations (lower $\ell_1$ Norm and effective rank, higher Top-$k$ energy and Hoyer sparsity). This localization to late layers supports the view that long-context difficulty primarily changes the \emph{final-stage computation} that consolidates evidence into the prediction-relevant subspace, rather than uniformly reshaping representations throughout the entire LLM.

LongReasonQA~\citep{ling2025longreason} is a long-context reasoning benchmark that systematically controls the amount of background context (e.g., 8K/16K/32K/64K) while keeping the core question and answer format fixed, enabling a clean test of how increasing context length affects model reasoning and internal representations as ~\autoref{fig:longreason_example_8k}.

\begin{figure*}[t]
    \centering
    \begin{tcolorbox}[width=\linewidth, title={\textbf{LongReason (8K) example (truncated background)}}]
    \footnotesize
    \textbf{Background (truncated):}\\
    \ttfamily
    Following the destruction of the house and its contents, Parks filed an insurance claim ...\\
    \quad \vdots \\
    ... Understanding this ejection process is essential for comprehending how ``Omega'' achieves its propulsion objectives.\\
    \normalfont
    \vspace{0.6em}

    \textbf{Question:}\\
    In the context of space station ``Omega,'' what does ``these fluids'' refer to?\\
    \vspace{0.3em}

    \textbf{Choices:}\\
    (A) Wave modulator \quad
    (B) Ion converter \quad
    (C) Pulsed high-energy currents \quad
    (D) Ion streams\\
    \vspace{0.3em}

    \textbf{Gold answer:} (D)
    \end{tcolorbox}

    \caption{\textbf{A LongReason 8K instance.} We show one example used in our long-context evaluation. For readability, the background is truncated (head + tail); the full context is available in the dataset.}
    \label{fig:longreason_example_8k}
    \vspace{-1em}
\end{figure*}

\subsection{Pretraining data construction details}
\label{apps:pre}

\begin{figure*}[t]
    \centering
    \begin{tcolorbox}[width=\linewidth, title={\textbf{LongReason (8K) example (paraphrased \& truncated)}}]
    \footnotesize

    \textbf{Background (paraphrased, truncated):}\\
    \ttfamily
    Note that adblockers might block our captcha, and other functionality on BHW so if you don't see the captcha or see reduced functionality please disable adblockers to ensure full functionality, note we only allow relevant management verified ads on BHW.\\
    \quad \vdots \\
   Briner’s observations confirm creationist Ice Age studies. It is most reasonable, given the Bible’s history, that there was a single Ice Age, and that the largest extent of ice lasted for only hundreds of years, beginning around 2400 BC. It is possible that the period of time wherein the “lion’s share” of melting apparently occurred was the only time that it occurred. The remaining thousands of years would therefore be a product not of data, but of the assumption of deep time.\\
    \normalfont
    \vspace{0.6em}

    \textbf{Question:}\\
    Which of the following, if true, does not weaken the viewpoint about the restoration of the jade pine forest in the country of Novaland after a chemical spill?\\
    \vspace{0.3em}

    \textbf{Choices:}\\
    (A) The pollution in the forest may have reduced to a level where some trees can survive. \\
     (B) Apart from artificially planted species, wild plants have also appeared. \\
     (C) Certain specific pollutants have been removed, but the heavy metal content in the forest soil is still higher than before.\\
     (D) The types and numbers of birds appearing in the forest are still very few, not reaching the scale before the pollution.\\

    \vspace{0.3em}

    \textbf{Gold answer:} (B)
    \end{tcolorbox}

    \caption{\textbf{A LongReason 8K instance (paraphrased).} We follow the same presentation format as our long-context example: truncated background + multiple-choice inquiry + gold label.}
    \label{fig:longreason_example_ecorestore_8k}
\end{figure*}

\subsubsection{Model and training setup.}
We train a randomly initialized, decoder-only Transformer following the LLaMA architecture.
Given a configuration \texttt{llama-$L$-$H$}, we set the hidden size to $d = 64H$, the number of layers to $L$,
the number of attention heads to $H$, and the MLP intermediate size to $2d$. The model is trained with a lightweight character-level tokenizer whose vocabulary consists of :\texttt{\{Q, P, 0--9, space, newline, '-', '?', <BOS>, <EOS>, <PAD>, <UNK>\}}.
Training data are generated from a latent-rule knowledge graph; each sample is a single triple serialized as \texttt{``Q\{h\} P\{r\} Q\{t\}\textbackslash n''} and optimized with the standard next-token prediction objective (labels identical to inputs, with padding labels masked by \texttt{-100}). We use the HuggingFace \texttt{Trainer} with per-device batch size 32, learning rate $10^{-4}$, cosine learning-rate
schedule, warmup ratio 0.2, and no weight decay. Unless otherwise specified, we train for \texttt{max\_steps} steps
(with optional BF16) and save a single checkpoint at the end of training.

\subsubsection{Synthetic knowledge graph generation.}
We construct a synthetic knowledge graph (KG), rather than a real-world one, because we can control the graph's size and complexity. By coupling a latent rule set with a graph growth process, we follow the setting of ~\citep{wang2025do}. We first sample a collection of acyclic relational rules, where each rule is represented as a relation chain $(r_0 \Leftarrow r_1, \dots, r_\ell)$ with $\ell \in [L_{\min}, L_{\max}]$. Next, we instantiate these rules on fresh entity nodes to form an initial backbone graph: for each rule, we create a directed path following the body relations $(r_1,\dots,r_\ell)$ and additionally insert the implied head edge
$r_0$ between the same endpoints. After building the initial backbone graph, we grow it to the desired size by adding one new entity at a time. For each new entity, we create $m$ connections to existing entities. The relation type of each new connection is not chosen completely at random; instead, it is restricted by a relation-adjacency map that records which relations tend to appear together in the sampled rules. Optionally, we bias the attachment toward already well-connected entities, producing a power-law degree distribution.

To inject deductive structure, we periodically apply a subset of rules to existing nodes via stochastic traversal. For each candidate head entity, we follow the rule body edges on the current graph (with probability $p_{\text{mcmc}}$ at each hop) to obtain reachable tail entities; whenever a rule fires, we add the implied triple $(h, r_0, t)$ and record its supporting rule(s). We refer to such implied edges as \emph{deductible} triples, while the remaining edges are treated as \emph{atomic} triples. We then form the training set as a mixture of atomic and deductible triples controlled by a deductible ratio $\rho$ and remove all non-training edges from the graph to prevent leakage. Finally, we design three evaluation splits: an in-distribution (ID) set sampled from the training triples (measuring
memorization), and two out-of-distribution (OOD) sets composed of deductible triples stratified by their shortest supporting
rule length, resulting in long-path and short-path OOD tests. Short-path OOD contains deductible triples whose shortest supporting rule is only 1–2 steps, while long-path OOD requires at least 3 steps. In our experiments, the short path is consistently harder than the long path. The key reason is that, under our synthetic generator, short rules are typically \emph{broad} while long rules are more
\emph{selective}. For \textsc{OOD-Short}, the rule body is easy to satisfy, so the same query $(h,r)$ often admits multiple
deducible tails (i.e., several $t$'s are consistent with the rules and the current graph), creating substantial ambiguity
and stronger competition among plausible answers. In contrast, \textsc{OOD-Long} imposes stricter compositional constraints, so far fewer paths match the full chain, and the set of valid deducible tails is usually much smaller (often close to a unique solution). As a result, the long-path split is easier, whereas the short-path split is harder in our setting. The structure of the knowledge graph construction can be found in ~\autoref{tab:synthetic_kg}.

\begin{table*}[t]
\centering
\small
\setlength{\tabcolsep}{5.3pt}
\renewcommand{\arraystretch}{1.10}
\begin{tabular}{|p{3.6cm}|p{10.2cm}|}
\Xhline{1.2pt}
\rowcolor{CadetBlue!20}
\textbf{Stage} & \textbf{Synthetic KG construction (as implemented)} \\
\Xhline{1.2pt}

\textbf{Universe} &
Entities are symbolic nodes $Q_0, Q_1, \ldots, Q_{n-1}$; relations are $P_0, P_1, \ldots, P_{n_r-1}$. \\
\hline

\textbf{Rule sampling} &
Sample a pool of candidate rules as relation chains:
$(r_0 \Leftarrow r_1, r_2, \ldots, r_\ell)$ with $\ell \in [L_{\min}, L_{\max}]$.
Optionally apply length-weighted sampling with temperature $\tau$ (shorter rules sampled more often).
Acyclicity is enforced via a dependency graph over relations. \\
\hline

\textbf{Backbone graph initialization} &
Construct an initial directed multigraph by instantiating each sampled rule on fresh entity nodes:
create a path $Q \xrightarrow{r_1} \cdots \xrightarrow{r_\ell} Q'$ and add the implied edge
$Q \xrightarrow{r_0} Q'$.
Maintain per-relation entity pools (``repeated entities'') for later attachment. \\
\hline

\textbf{Graph growth (structural prior)} &
Iteratively add new nodes until reaching $n$ entities.
Each new node attaches $m$ edges by sampling relations constrained by an adjacency map derived from rule co-occurrence
(also supporting a power-law attachment mode). \\
\hline

\textbf{Deductible edge injection} &
Periodically enumerate candidate heads $h$ and attempt to apply a subset of rules (``deductible rules'') by stochastic traversal:
follow each rule body with probability $p_{\text{mcmc}}$ to obtain reachable tails $t$; if successful, add the implied triple
$(h, r_0, t)$ as a \emph{deductible} edge and record its supporting rule(s). \\
\hline

\textbf{Train graph pruning} &
Form training triples as a mixture of \emph{atomic} edges (not derivable) and \emph{deductible} edges (derivable),
controlled by deductible ratio $\rho$.
All edges not selected for training are removed from the graph to avoid leakage. \\
\hline

\textbf{Test set design (difficulty)} &
Three evaluation splits are constructed:
(i) \textbf{ID/Easy}: sampled directly from training triples (memory);
(ii) \textbf{OOD-Long/Medium}: deductible triples supported by \emph{long} rules (e.g., $\ge 3$ steps);
(iii) \textbf{OOD-Short/Hard}: deductible triples supported by \emph{short} rules (e.g., 1--2 steps).
For OOD splits, ``seen tails'' are collected from alternative deductive paths for harder negative sampling. \\
\hline

\textbf{Text serialization} &
Each triple is serialized as a character-level sequence:
\texttt{Q\{h\} P\{r\} Q\{t\}} (plus newline), enabling a lightweight tokenizer and controlled vocabulary. \\
\Xhline{1.2pt}
\end{tabular}
\caption{Synthetic knowledge graph generation via a latent-rule process: rule sampling, graph growth, deductible edge injection, and difficulty-controlled evaluation splits.}
\label{tab:synthetic_kg}
\end{table*}

\subsubsection{More details about the learning dynamic experiment.}

\label{app:learning_dynamics_details}

~\autoref{fig:learning_dynamics_full} provides a more complete view of the learning dynamics discussed in ~\autoref{sub:learning_dynamics}. In addition to the $\ell_1$ trend shown in the main text, we track scale-normalized sparsity metrics and accuracy throughout pretraining on the controlled graph environment. The model is trained only with the next-token prediction objective, and checkpoints are evaluated on three splits: ID (Easy), OOD (Medium), and OOD (Hard).

The results reveal the same two-phase pattern from a complementary perspective. During the early stage of training, Top-5\% and Top-10\% Energy increase while Effective Rank decreases, indicating that the representation becomes increasingly concentrated on a smaller set of active dimensions. This corresponds to an initial feature-selection or pruning phase. After this early concentration phase, Top-$k$ Energy gradually decreases and Effective Rank recovers, while accuracy continues to improve or stabilize, especially on the ID split. This suggests a later feature-consolidation phase in which the model no longer relies on a few sharp coordinates, but distributes useful information across a broader and more stable representation. Importantly, because Top-$k$ Energy and Effective Rank are concentration-based metrics, this trend is not merely an artifact of raw $\ell_1$ scale changes. Together with ~\autoref{fig8}, these results support the interpretation that sparsity evolves systematically during training rather than appearing as a static property of the architecture.

\begin{figure*}[t]
    \centering
    \includegraphics[width=0.8\textwidth]{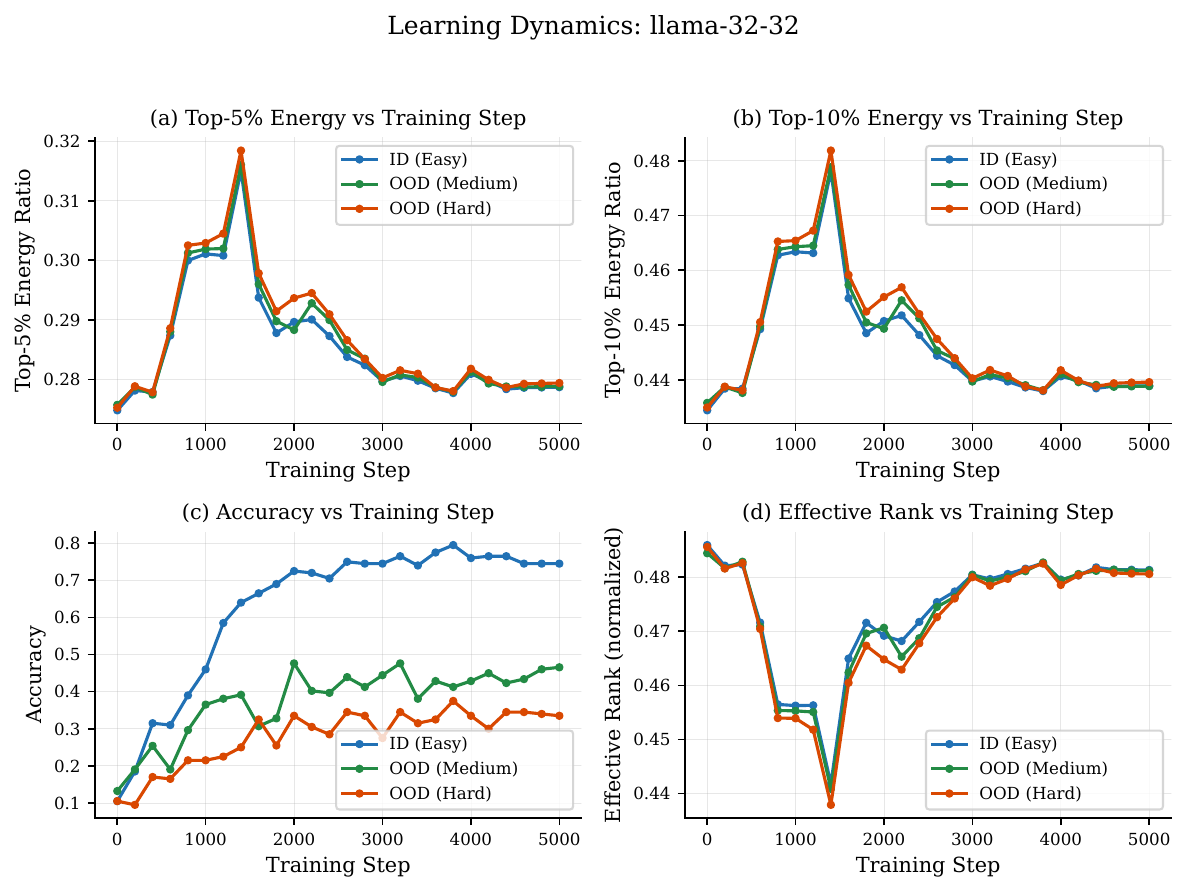}
    \caption{
    \textbf{Learning dynamics across sparsity metrics and accuracy.}
    We track Top-5\% Energy, Top-10\% Energy, accuracy, and normalized Effective Rank across training checkpoints for the controlled graph experiment. Early training shows increasing energy concentration and decreasing effective rank, while later training shows partial recovery of effective rank and stabilization of accuracy, supporting the two-phase feature-selection and feature-consolidation interpretation.
    }
    \label{fig:learning_dynamics_full}
    \vspace{-1em}
\end{figure*}

\begin{figure*}[t]
    \centering
    \includegraphics[width=0.8\textwidth]{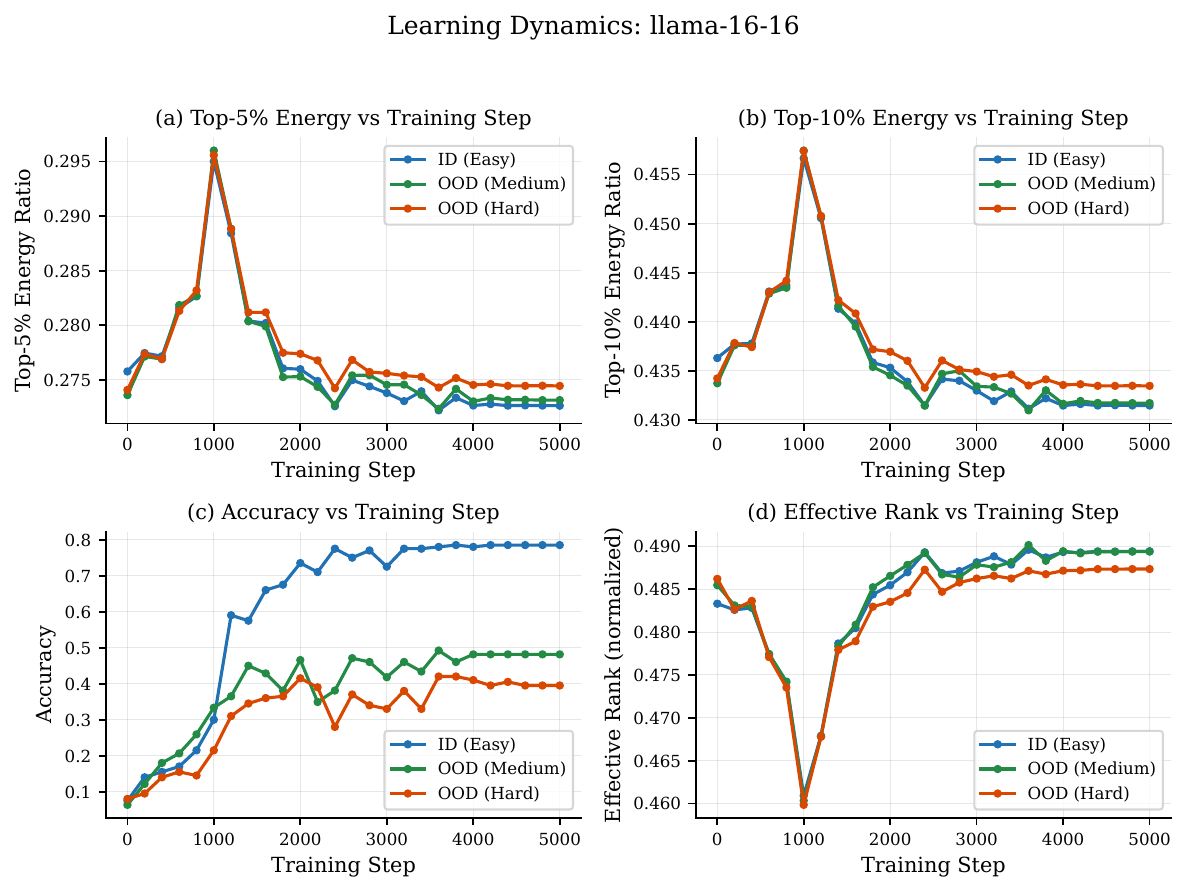}
    \caption{
    \textbf{Learning dynamics for Trans-16-16.}
    We track Top-5\% Energy, Top-10\% Energy, accuracy, and normalized Effective Rank across training checkpoints in the controlled graph environment. Compared with the larger Trans-32-32 model, Trans-16-16 exhibits a similar early concentration phase, where Top-$k$ Energy increases and Effective Rank decreases, followed by a relaxation phase in which energy concentration weakens and Effective Rank recovers. This provides additional evidence that the sparsity dynamics are tied to feature selection and consolidation during training, rather than being an artifact of a single model scale.
    }
    \label{fig:learning_dynamics_16_16}
    \vspace{-1em}
\end{figure*}
\subsection{An Extended Discussion and Future Work}
Our investigation establishes a robust empirical law: \textit{The farther the distribution shift, the sparser the representation.} Beyond the raw observations, this section synthesizes the underlying mechanisms and the broader implications of this phenomenon.

We hypothesize that explicitly training for ``OOD-like'' sparsity could improve model robustness and generalization, effectively simulating the ``active suppression'' mechanism during the learning phase. Our experiments utilized dense architectures (Llama series). Mixture-of-Experts (MoE) models introduce structural sparsity by design. Investigating whether the ``OOD-induced sparsity'' phenomenon persists in MoE architectures, or whether the routing mechanism absorbs the shift, would be critical for generalizing our findings to the next generation of sparse foundation models. Our study focused on reasoning tasks where the model generally attempts to solve the problem. An intriguing direction is to analyze sparsity patterns during hallucinations. Does a model exhibit \textit{low} sparsity (high confidence/confusion) when it hallucinates, or does it exhibit \textit{extreme} sparsity (information collapse)? Establishing a correlation between sparsity signatures and factual errors could lead to lightweight, white-box hallucination detection methods that do not require external fact-checking.

\subsection{Limitations}
\label{app:limitations}
While our findings offer significant insights into the mechanism of LLMs, we acknowledge several limitations in our current study. First, our proposed sparsity metrics (e.g., Top-$k$ Energy) rely on white-box access to the model's last hidden states, rendering the method inapplicable to closed-source proprietary models like GPT-4o. Second, our experimental scope is primarily centered on logical reasoning (MATH-500) and long-context QA (LongReason); while we hypothesize that the sparsity-difficulty relationship generalizes to other domains, its applicability to tasks such as creative writing or multilingual processing remains to be fully verified. Finally, although SG-ICL improves reasoning accuracy, computing sparsity scores for large demonstration pools introduces pre-processing overhead compared to random sampling, though this cost can be effectively amortized by pre-computing and caching difficulty scores.

\section{The Use of Large Language Models (LLMs)}
We used LLMs (e.g., ChatGPT) solely as general-purpose tools for
\textbf{(i)} language editing and polishing of manuscript drafts and
\textbf{(ii)} coding assistance for minor boilerplate (e.g., plotting and small utilities). All outputs were fully reviewed, modified, and tested by the authors. All research ideas, algorithmic designs, experiments, datasets, analyses, and conclusions were conceived and validated by the authors. LLMs were \textbf{not} used to produce results, annotations/ground truth, or methodological decisions.
The authors take full responsibility for the content of this paper.
\end{document}